\documentclass[journal]{IEEEtran}
\usepackage{enumerate}
\usepackage{mathtools}    
\newcommand\numberthis{\addtocounter{equation}{1}\tag{\theequation}}
\usepackage{algorithm}
\usepackage{algorithmic}
\usepackage{color,soulutf8}
\usepackage{xcolor}

\usepackage{amsmath} 
\usepackage{amsthm}
\usepackage{bm}
\usepackage{graphicx}
\usepackage{amsfonts}
\usepackage{relsize}
\usepackage{float}
\usepackage{subfigure}
\usepackage{subfigmat}
\usepackage{etoolbox}

\newtheorem{theorem}{Theorem}
\patchcmd{\subfigmatrix}{\hfill}{\hspace{0.8cm}}{}{}
\ifCLASSINFOpdf
\else
\fi
\begin{document}
\title{Spatiotemporal Feature Learning for Event-Based Vision}
%
%
%
%

\author{Rohan~Ghosh,~\IEEEmembership{Student Member,~IEEE,}
        Anupam~K.~Gupta,
        Siyi~Tang,
        Alcimar B. Soares,
        and~Nitish~V.~Thakor,~\IEEEmembership{Life~Fellow,~IEEE}
\thanks{}}

\IEEEtitleabstractindextext{

\begin{abstract}
Unlike conventional frame-based sensors, event-based visual sensors output information through spikes at a high temporal resolution. By only encoding changes in pixel intensity, they showcase a low-power consuming, low-latency approach to visual information sensing. To use this information for higher sensory tasks like object recognition and tracking, an essential simplification step is the extraction and learning of features. An ideal feature descriptor must be robust to changes involving (i) local transformations and (ii) re-appearances of a local event pattern. To that end, we propose a novel spatiotemporal feature representation learning algorithm based on slow feature analysis (SFA). Using SFA, smoothly changing linear projections are learnt which are robust to local visual transformations. In order to determine if the features can learn to be invariant to various visual transformations, feature point tracking tasks are used for evaluation. Extensive experiments across two datasets demonstrate the adaptability of the spatiotemporal feature learner to translation, scaling and rotational transformations of the feature points. More importantly, we find that the obtained feature representations are able to exploit the high temporal resolution of such event-based cameras in generating better feature tracks.
\end{abstract}

\begin{IEEEkeywords}
Event-based vision, Feature Learning, Slow feature analysis
\end{IEEEkeywords}}

\maketitle

\IEEEdisplaynontitleabstractindextext

%

~\\

\IEEEraisesectionheading{\section{Introduction}\label{sec:introduction}}


\begin{figure}[h]
    \centering
\includegraphics[width=0.5\textwidth]{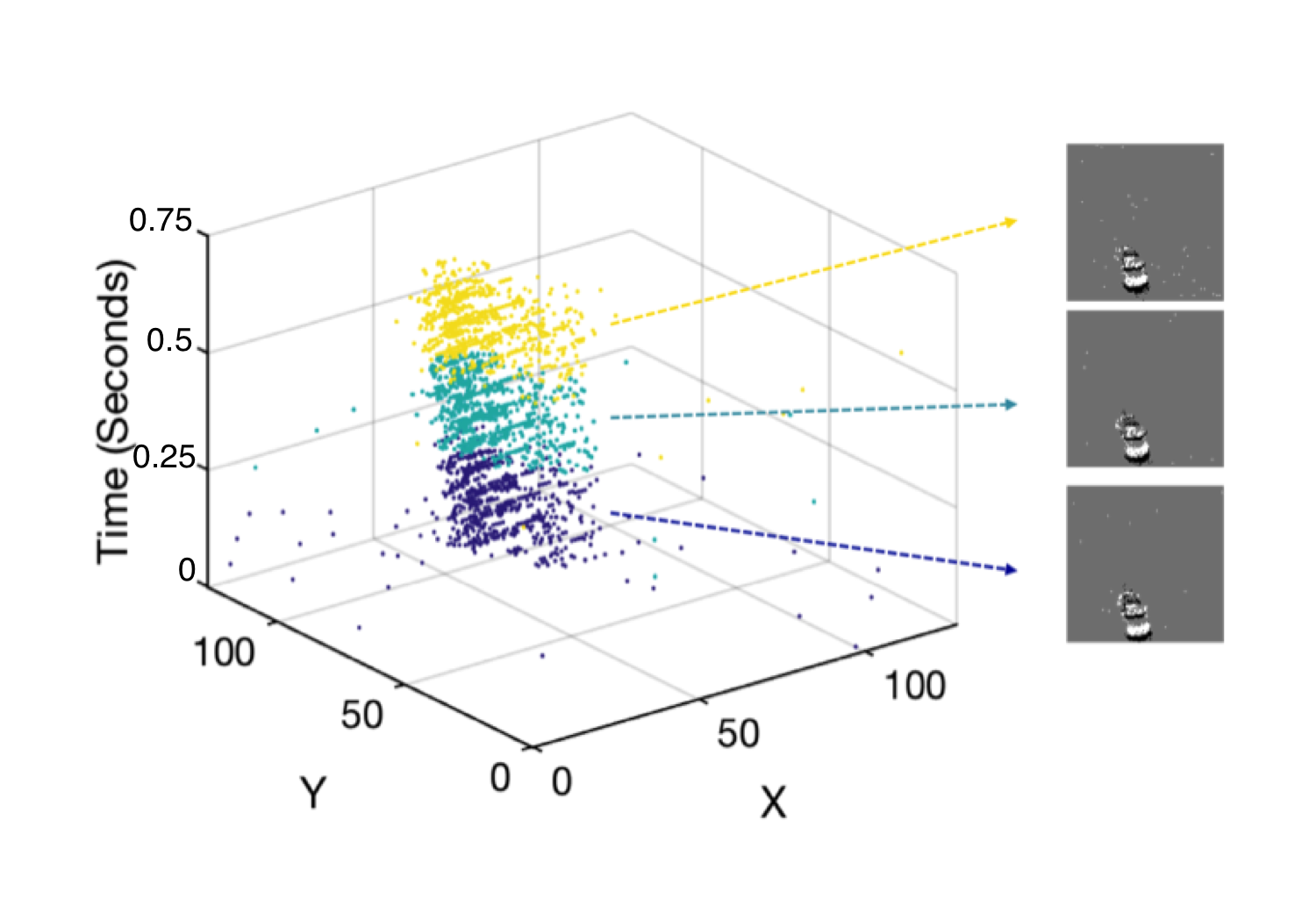}
\caption{Example spike-event data from the DVS sensor in response to a moving car on a road. Pixels produce spike-events as points in 3D $(X,Y,T)$ space, indicating the change of intensity at $(X,Y)$ at time $T$. The change is encoded as either incremental (white) or decremental (black), as observed in the images to the right. The images are created by binning a number of spike-events (color-coded). }
\label{fig:traffic_data}
\end{figure}
Feature extraction is an integral part of numerous applications in computer vision. It is common practice to compress the high dimensional information contained in digital images and videos into low dimensional features. This process aids in reducing the computational burden on subsequent algorithmic modules by only channelling relevant and interesting information.
Feature descriptors obtained from local image regions around \textit{interest points} (or feature points, e.g. \cite{lowe_keypoint}) act as robust transformation-invariant markers. Extraction and matching of feature descriptors are fundamental to applications such as simultaneous localization and mapping (SLAM) \cite{orb_slam}, depth estimation from stereo images \cite{stereo_surfaces}, object recognition \cite{hog_human}, motion segmentation \cite{ssc_original}, object tracking \cite{tracking_feature} and action recognition \cite{action_features}. 

Conventional, frame-based sensors perform synchronous sampling of the entire pixel array. In contrast, event-based, neuromorphic sensors only acquire information when the light intensity at a pixel shows a threshold amount of \textit{change}. Some examples of such event-based cameras include DVS \cite{dvs_original}, ATIS \cite{atis_original} and DAVIS \cite{davis_original}, which are primarily characterized by-
\begin{itemize}
\item Asynchrony: Each pixel outputs an event when the pixel intensity changes by a predetermined threshold. Such events are often denominated as \textit{spike-events}, due to their resemblance to neuronal spikes. 
\item Low data redundancy: Pixel intensity changes usually only occur at object edges, thereby, limiting the number of spike-events, and reducing data redundancy.
\item Low latency: The spike times are accurate to $ \sim 10  \mu $s, enabling almost continuous visual information acquisition. Furthermore, each spike is available for processing within 10 $\mu$s. 
\end{itemize}
 
 The data generated from event-based sensors are inherently spatiotemporal in the form of sparsely located spike-events. Intensity changes at a pixel can occur either due to global illumination changes or edge motion. Figure \ref{fig:traffic_data} (a) is an example of the point cloud generated from accumulating spike-events over a certain interval of time. In this instance, the events are generated due to the motion of a car, which is easier to spot in images created by the events (Figure \ref{fig:traffic_data} (b)). 
 
Extraction of features from such sparsely located spike-events presents an interesting challenge. First, one can observe a lot of noisy spike-events. Second, many of the edges have a very weak presence, due to their motion direction being parallel to the edge. Third, in spite of a very high time resolution of event-based sensors, the spike-event times will not be microsecond precise, decreasing the \textit{effective} temporal resolution by a order of 100 \cite{temporal_ryad}. Noisy spike-events are usually produced at a constant (small) rate, caused by Poisson shot noise fluctuating the light intensities at the pixel photoreceptors \cite{dvs_original}. The pattern of events produced in response to a moving visual feature (e.g. a corner) will thus exhibit a certain degree of variability each time the feature is presented. However, in spite of such issues, it is possible to find edge orientation, edge motion and detect corners with more precision and accuracy than a frame-based camera \cite{corner_event}, owing to the constructive use of the high temporal resolution of events.

Local feature descriptors with static images typically consider local image regions which have properties that distinguish them clearly from their neighbors \cite{local_features}. Often these desirable properties emerge from a non-uniform pixel intensity structure. Interesting and informative image regions usually form around spatial edges created by local intensity differences. Since event-based sensors only respond to intensity changes, the information contained in spike-events is naturally indicative of interesting features. As the spike-events essentially trace the moving edges in the scene, the features can be intrinsically tuned to edge shape and motion. 



Since there is no pixel-intensity information available from event-based sensors, the spike-timing information is instead used to generate potent feature representations. To that end, time-surfaces \cite{hots_garrick} are often employed as a local feature construction method. A time-surface is a 2-dimensional representation of the last time of occurrence of a spike-event at each pixel location in a local spatial neighborhood. A feature estimation framework comprising a hierarchical approach to encoding time-surface information was described in \cite{hots_garrick}. Features which occurred more frequently throughout the data, were prioritized. Neural network based methods for learning features were used in \cite{bichler_snn_car}. There, hierarchies of spiking neural networks were used to learn and extract temporally correlated features. Yet another way to extract features is through a predictive model of the spike-event data in a local spatiotemporal region, which was done in \cite{echostate_features}. Features were defined as discrete classes of spatiotemporal predictors which used echo-state networks. The method proposed there exploited spatiotemporal correlations between spike-events for better prediction of spike-event patterns, and used them for subsequent feature creation.  

It is evident from these works that capturing spatiotemporal correlations present in the spike-event patterns can be used to learn features useful for recognition and tracking applications. However, an aspect of feature construction absent in prior work is robustness to the various ways a certain visual feature could be presented to the camera. For example, a local feature of an object might view differently depending on the scale, position and pose of the object. Since event-based cameras respond to motion, an additional aspect of variation would be the speed of the object itself. In this work, we primarily wish to address the problem of estimation of features which are robust to such transformations, and are inherently spatiotemporal in nature. 

Instead of computing handcrafted feature representations and subsequently using them for feature matching, an alternative, used in this work, is to \textit{learn} optimal representations from matching pairs of interest points which have similar patterns. The local pixel-intensity pattern around an interest point (e.g. a corner) can show significant changes in time with change of scale and pose. It is harder to perform feature matching when the object has undergone significant changes in pose and scale. In contrast, matching feature points between two consecutive viewings of an object is much easier, where those changes can be assumed to be less pronounced. Motivated by this, some methods in frame-based vision report a \textit{track and learn} strategy \cite{unsupervised_wang,siamese_face,tld_tracking,tlm_tracking}, where multiple pairs of tracked interest points across frames are used to learn invariant feature representations. The inherent assumption is that features computed from two consecutive viewings of an interest point (or an object) must not show large differences. Such global object features may be computed with a convolutional neural network architecture \cite{siamese_face} as in\cite{unsupervised_wang}. The counterpart for local feature points is the tracking-learning-matching technique proposed in \cite{tlm_tracking}, where invariant local features are learnt.



Here, we propose a technique for learning invariant local feature representations from event data, similar to \cite{tlm_tracking} and \cite{sfa_original}. While the features obtained in those works were primarily spatial, in our method they are spatiotemporal, and therefore capture both shape and motion properties present in the local spatiotemporal event pattern. The contributions of this work are three-fold.

First, a novel feature learning algorithm is described, which prioritizes smooth but informative representations. We use a novel approach to obtain feature representations from event-based data, based on constraining the changes in the feature values to the temporal evolution of a spatiotemporal event pattern. Specifically, slow feature analysis \cite{sfa_original} is used to arrive at the spatiotemporal weights (analogous to receptive fields) required for feature generation. Contrary to any previous work, we directly learn features from raw event-data, with only enforcing a slowness constraint on the feature value. Slow feature analysis (SFA, \cite{sfa_original}) proposes a feature learning paradigm where only features changing slowly in time are prioritized, as they are more likely to be robust to transformations. SFA has been successfully applied to frame-based vision, where the method learns features which are robust to various kinds of transformations, while being informative about the visual pattern. Here we propose a spatiotemporal domain SFA approach, that learns features ground-up from raw spike-event data.

Second, a feature point tracking algorithm is proposed, just to evaluate the effectiveness of trained features. This helps assess whether the extracted feature representations are spatiotemporally smooth while preserving the identity of an event-pattern. The main goal of our results is to demonstrate the ability of the features to specialize towards the type of spatiotemporal activity occurring in the training dataset, while simultaneously showing decent generalization performance to other datasets. Through our experiments, we see the extracted spatiotemporal features learn invariances to translation, rotation and scaling based transformations.  

Lastly, we find that the optimized spatiotemporal features come out to have quite interesting properties, roughly analogous to visuo-cortical receptive fields. Some of them encode spatial edge orientation and others seemed to specialize towards motion direction of the edges, showing both spatial and spatiotemporal features are extracted.


\section{Why Spatiotemporal?}

Spatial features, for e.g. from SIFT, \cite{lowe_keypoint} are usually derived from the edge orientation distribution in a local region. This helps basic feature descriptors differentiate between corners, curves, and straight lines. Spatiotemporal features encode temporal motion dynamics in addition to edge geometry. Therefore, the spatiotemporal feature representations corresponding to two different motion directions of the same spatial pattern can be significantly different. A spatial feature descriptor is usually robust to illumination changes and distortion (e.g from change of pose) of the feature input. A spatiotemporal feature descriptor must additionally be invariant to slight motion distortions, i.e. changes in motion direction or magnitude. Due to such properties, we can underline scenarios where computing spatiotemporal feature descriptors can be useful and/or necessary.

\begin{itemize}
\item \textbf{Feature Point Tracking (and all relevant applications):} Random motion patterns forego any sense of temporal regularity, which limits the effectiveness of spatiotemporal features for this scenario. On the other hand, when motion can be assumed to be smooth, point trackers based on spatiotemporal features can provide better tracking capabilities as the spatiotemporal pattern surrounding the feature point shows little change over time. 
\item \textbf{Action Recognition:} Human actions constitute motion patterns which register a unique spatiotemporal footprint. This makes our method ideal for application in such a scenario.
\item \textbf{Deformable object tracking:} Spatial descriptors pertaining to an object will change with time as it deforms. On the other hand, spatiotemporal features invariant to such deformations will show greater robustness while tracking. 
\end{itemize}

In this paper, we test our spatiotemporal features on the problem of tracking feature points over time. To achieve this robustly, the features need to be invariant to the transformations to the feature point happening through time, and simultaneously selective enough, such that the tracker does not stray away from the intended feature point. Therefore, we expect that better spatiotemporal features will demonstrate better point tracking performance. Our spatiotemporal feature learning method proposed here does not explicitly consider the physical significance of the spike-events, and should extend just as well to a temporal stream of point-clouds (for e.g. in structure-from-motion estimation problems). The primary assumption in the feature learning method proposed here, is that more often than not a spatiotemporal feature will re-appear over time, most likely in the vicinity of its earlier appearance. This assumption mainly applies to event-based cameras, as almost continuous sampling of the scene ensures "continuity" in the trajectory of a visual feature.

\section{Methods} \label{sec:methods}

\subsection{Method Summary} \label{sec:m_summary}

\begin{figure}[h] 
    \centering
    \includegraphics[width=0.48\textwidth]{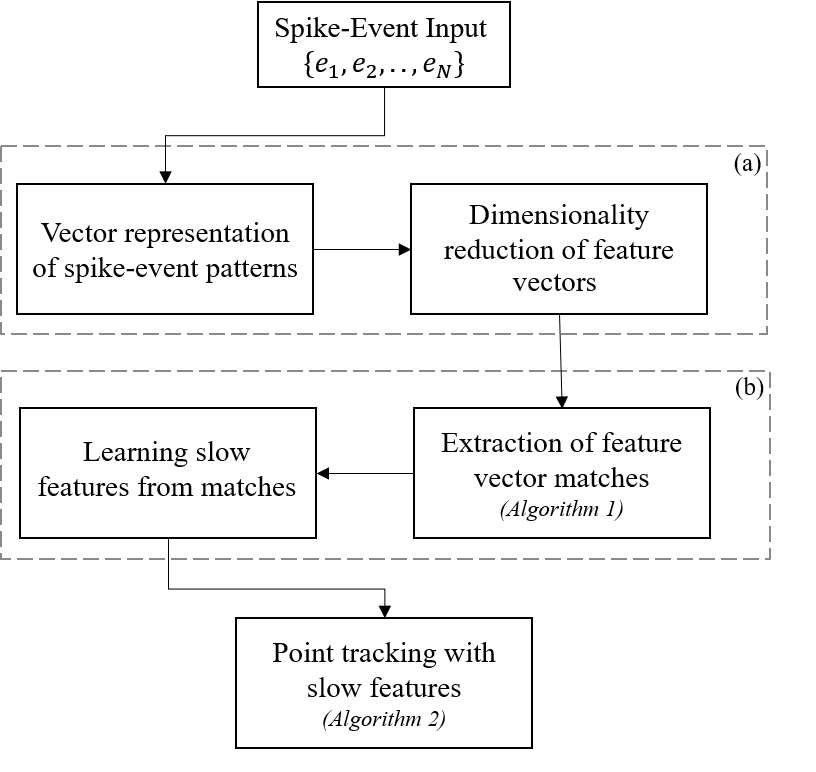}
    \caption{A flow chart summary of the methodology proposed in this paper, including the feature match extraction and point tracking algorithms. The approach can be divided into two parts; steps in (a) generate an initial PCA based feature representation which is later used in (b) to learn more robust features using SFA.}
    \label{fig:flow_chart}
\end{figure}
Fig. \ref{fig:flow_chart} presents a flow chart of the steps involved in our proposed methodology. The unsupervised feature learning method can be roughly divided into two parts, as demonstrated in the figure. We summarize each part as follows:
\begin{enumerate}
\item \textbf{Part (a)}: Here, an initial feature vector representation is generated for each event's box-neighbourhood. It simply consists of vectorizing the box neighbourhood while preserving the spatiotemporal topology. This culminates in pruning the feature vector to only account for the most informative components, generating a smaller, more dense feature representation. Therefore, such a sparse-to-dense conversion of the spike-event space is an essential part of the feature extraction and learning process, and will be used throughout.  
\item \textbf{Part (b)}: Here the feature representations are further enhanced by obtaining more stable linear projections using the track and learn approach. First, using the feature representations obtained in part (a), pairs of feature matches are generated throughout the data through tracking each feature point (Algorithm 1). Subsequently, slow feature analysis is performed to learn more robust projections which vary less within a tracked pair.  
\end{enumerate}

Once the final SFA based features are obtained, they are used in the point tracking algorithm (Algorithm 2). Next, we detail each step of our approach to learn and computing spatiotemporal features in the following sections.

\subsection{Box Neighborhood creation} \label{sec:background}

Spike-events generated from an event-based camera can be denoted by the set
\begin{eqnarray} \label{eq:stime}
E=\{ e_i \}_{i=1}^N \quad
\mbox{where} \quad e_i=(x_i,y_i,t_i,p_i).
\end{eqnarray}

Each spike-event is described by its spatial location $(x_i,y_i)$, time of occurrence $t_i$ and polarity $p_i$. In this work, we do not consider polarity information in the learning and computation of features. Polarity adds additional variation to the spatiotemporal patterns, and therefore will deter the robustness of the computed spatiotemporal features, due to this additional factor of variation. Without polarity, the spike-events can be simply visualized as points in a three-dimensional space of spike location $(x,y)$ and time $t$. We define a \textit{box neighborhood} around each spike-event $e_i$ as the set of neighboring spike-events which belong to a cuboidal region of size $(\Delta x \times \Delta y \times \Delta t)$ around $e_i$. We denote the box neighborhood around an event $e_i$ by the set $B_{E}(e_i,\Delta x, \Delta y, \Delta t)$.  An illustration of this step can be found in Fig. \ref{fig:dimreduce}(a).%

\subsection{Neighborhood Spike-Count Matrix}

\begin{figure*}[h]
    \centering
    \includegraphics[width=1\textwidth]{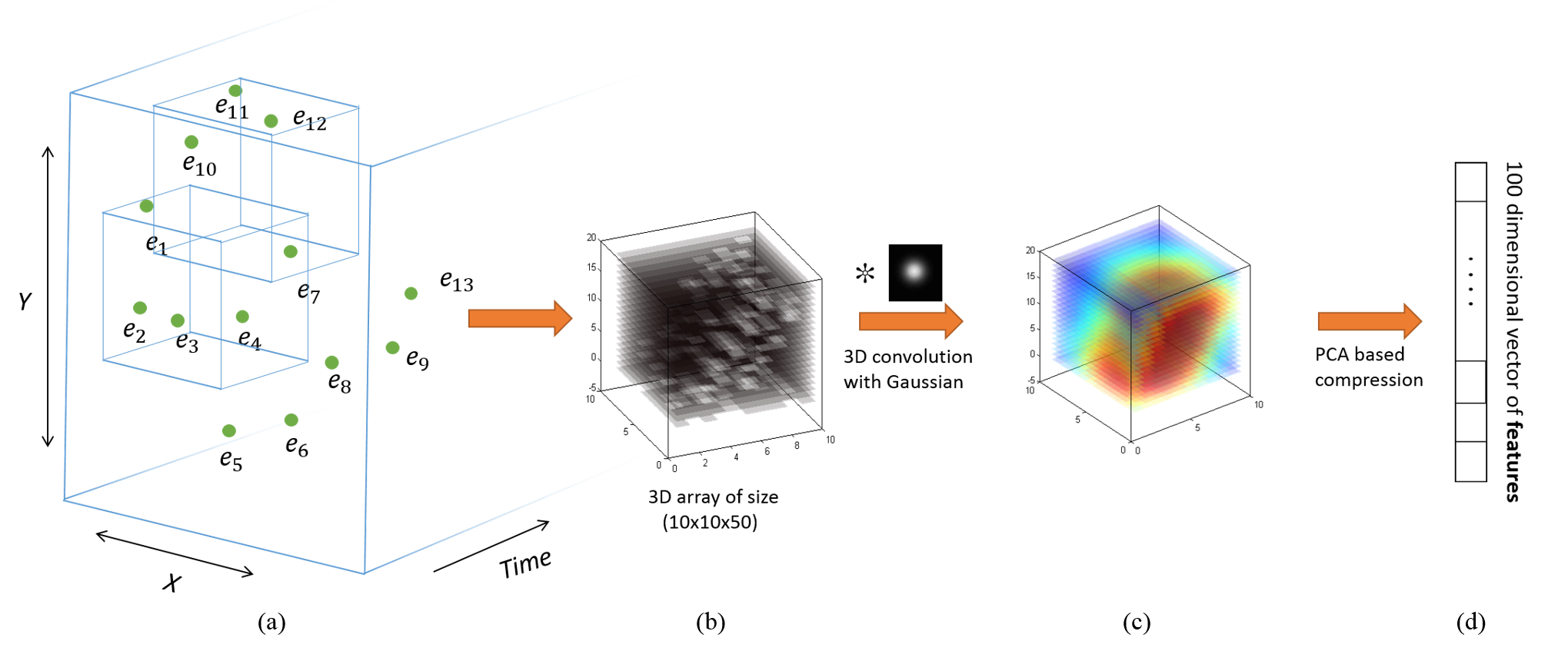}
    \caption{An illustration of the steps involved in feature compression. (a) demonstrates the creation of box neighbourhoods with spike-event data; (b) example of a three dimensional spike count matrix for parameter values $\Delta T=100$ ms, $M=50$, and $a=10$; (c) corresponding smoothed matrix with large smoothing parameters $\sigma_x=\sigma_y=\sigma_t=10$ to emphasize the coarse distribution of events within the box neighbourhood; (d) schematic showing the pruning of the $5000$ dimensional matrix to a much smaller, 100 dimensional feature vector using PCA (as decided by the 95\% variance preserving rule).}
    \label{fig:dimreduce}
\end{figure*}

Here a sparse 3D matrix is obtained, which shows a spatiotemporal voxel histogram distribution of spike-events within the box-neighborhood region around a spike-event $e_i$. The events in $B_{E}(e_i,\Delta x, \Delta y, \Delta t)$ are therefore used to create this matrix, which is representative of the distribution of spike-events around $e_i$. We set $\Delta x = \Delta y  = a$ and $\Delta t = T$. This volume is partitioned into a three-dimensional voxel grid, with each voxel of size $(1 \times 1 \times T/M)$. $M$ denotes the number of partitions along the time dimension. We empirically study the impact of $T$ and $M$ on feature quality in Sec. \ref{sec:results}. Expectedly, we find that too much precision (M$>$100) causes performance to drop considerably, as does too little (M$<$10).

With this, we define the neighborhood spike-count matrix $\mathbf{C}(e_i)$ as a 3D matrix of size $(a \times a \times M)$, such that each element of $\mathbf{C}(e_i)$ contains the number of spike-events inside a voxel within $B_{E}(e_i,a,a,T)$ (See Fig. \ref{fig:dimreduce}(b)). 
$\mathbf{C}(e_i)$ is flattened to form a one dimensional feature vector $\mathbf{c}(e_i)$ containing $d = a^2 M$ elements. Note that for our experiments $d$ can be typically large ($500-5000$), depending on the number of temporal partitions $M$. Choosing a higher value of $M$ preserves the spike time details, but can make $\mathbf{c}(e_i)$ very high dimensional, and vice-versa.

Since event timing information is essential for obtaining accurate feature representations, $M$ should be kept sufficiently high. Instead, we introduce a dimensionality reduction step for $\mathbf{c}(e_i)$ in the following section, which reduces feature dimension while preserving event time information to a greater degree.  

\begin{algorithm} 
 \caption{Sparse-to-dense conversion of the local spike-event pattern around $e_i$}
 \label{alg:sparse_to_dense}
\begin{algorithmic} 
 \REQUIRE Given event $e_i = (x_i,y_i,t_i)$, 3D Gaussian covariance matrix $\Sigma$, box-neighborhood size: $(a\times a \times T)$, temporal partition: $M$, Top-k PCA weight matrix $W_{PCA} = [w_1, w_2,...,w_k]$
 \ENSURE Dense feature representation $F_{PCA} \left (e_i, T\right )$
 \STATE 1) Consider the box-neighborhood of size $(a \times a \times T)$ around $e_i$. \\ 
 \STATE 2) Create the spike-event count matrix $\mathbf{C}(e_i)$ , by voxel-wise histogramming of the spike-events in the box-neighborhood region. Each voxel is of size $(1 \times 1 \times T/M)$. \\
 \STATE 3) Smooth $\mathbf{C}(e_i)$ with the 3D Gaussian $\Sigma$, to generate the smoothed and flattened vector $\mathbf{c}_{\Sigma}(e_i)$. \\ 
 \STATE 4) Generate the final $k$ dimensional dense representation $F_{PCA} \left (e_i, T\right ) = \left(W_{PCA} \right)^T \mathbf{c}_{\Sigma}(e_i)$, as the projections of the first $k$ principal components. \\
\end{algorithmic}
\end{algorithm}

\subsection{Sparse-to-Dense Framework}\label{sec:pca_reduce}
The spike-count matrix is first convolved with a 3D Gaussian kernel, $\mathcal{N}(0,\Sigma)$. $\Sigma$ is a diagonal covariance matrix, containing values of $\sigma_x$, $\sigma_y$ and $\sigma_t$ along the diagonal. This smoothing step ensures that neighboring voxel locations in $\mathbf{C}(e_i)$ show high correlation. A more involved discussion on the necessity of this step is taken up in Section \ref{sec:discussion}. After convolution with the Gaussian kernel, the smoothed spike-count matrix is denoted as $\mathbf{C}_{\Sigma}(e_i)$.

$\mathbf{C}_{\Sigma}(e_i)$ is flattened to obtain the vector $\mathbf{c}_{\Sigma}(e_i)$.  Fig.~\ref{fig:dimreduce}(c) illustrates the smoothing step with a real example of spike-event data. Observe that over-smoothing could lead to loss of temporal and spatial information due to over-smoothing. In contrast small variance parameters retain temporal and spatial information, but generate a sparse vector. For all our experiments we keep the smoothing parameters unchanged at $\sigma_x=\sigma_y=\sigma_t=3$ \footnote{ Note that the unit of $\sigma_t$ is w.r.t the time dimension of each voxel, which in this case is $T/M$.} This choice of standard deviation was set empirically. 

The smoothed spike-count vector $\mathbf{c}_{\Sigma}(e_i)$ is obtained for all events $e_i$. To reduce dimensionality of  $\mathbf{c}_{\Sigma}(e_i)$, a principal components analysis is performed to choose only $k$ ($\leq d$) projections which preserve 95\% of the variance in the spike-count vector . As an example shown in Fig. \ref{fig:dimreduce}(d), 100 components are chosen for the feature representation based on this metric. The $k$ pca weights $\mathbf{W}_{PCA} = [w_1, w_2,...,w_k]$ are vectors of dimensionality $a^2M$ and therefore can be matricized\footnote{Matricization is referred to converting a $(1\times n^2)$ vector into its original 2D matrix form. Here we extend the definition to 3D matrices. Since each dimension of $w_i$ corresponds to a certain voxel within the box-neighborhood, the matricization of $w_i$ simply maps each element to its corresponding voxel, thereby creating a 3D matrix. Note that vectorization is the opposite.} to a 3D matrix of the same size as the spike-count matrix. If we let these matricized forms of $(w_1, w_2,...,w_k)$ be denoted as $(\mathbf{W}_1, \mathbf{W}_2,...,\mathbf{W}_k)$, then the product $w_j^{T}\mathbf{c}_{\Sigma}(e_i)$ can be rewritten as

\begin{equation} \label{eq:pca_feats}
SUM \left ( \mathbf{W}_j \circ \left (\mathbf{C}(e_i) \ast \mathcal{N}(0,\Sigma) \right ) \right ) .
\end{equation}

The $SUM(\mathbf{X})$ operator returns a scalar which is the sum of all elements of the matrix $\mathbf{X}$. The $\circ$ symbol is simply element-wise multiplication. The sequence of operations thus described will return a $k$ dimensional, linear compression of each spike-count matrix. It can be characterized as a PCA based feature representation. In practice, the computations in equation \ref{eq:pca_feats} must be repeated for each spike-count matrix $\mathbf{C}(e_i)$, including the convolution process with the Gaussian kernel. To avoid such repetitive convolution, we use the following Theorem. 

\begin{theorem}\label{thm:circconv}
Given 3D matrices $\mathbf{A}$, $\mathbf{B}$ and $\mathbf{C}$ where $\mathbf{C}$ is a 3D Gaussian kernel, $SUM \left(\mathbf{A} \circ \left (\mathbf{B} \ast \mathbf{C} \right) \right)= SUM \left ( \left(\mathbf{A} \ast \mathbf{C} \right) \circ \mathbf{B} \right)$ (Proof in Appendix). 
\end{theorem}

The above result allows us to re-order the convolution in equation \ref{eq:pca_feats}. The $j^{th}$ component of the PCA features can be thus rewritten as  

\begin{equation} \label{eq:pca_feats_swap}
SUM \left ( \left ( \mathbf{W}_j  \ast \mathcal{N}(0,\Sigma) \right) \circ \mathbf{C}(e_i) \right ),
\end{equation}

which is the inner product between the vectorized spike-count matrix $\mathbf{c}(e_i)$ and the vectorized form of the smoothed weight-matrix $\mathbf{W}_j  \ast \mathcal{N}(0,\Sigma)$, denoted as $w^{\Sigma}_j$. Thenceforth, for simplicity of notation, the PCA based representations on the box-neighborhood region of event $e_i$ is summarized as

\begin{equation} \label{eq:pca_features}
\mathbf{pc}_i = F_{PCA} \left (e_i, T\right ).
\end{equation} 

As mentioned before, $T$ is the size of the box-neighborhood cuboid along the time dimension. In this work we keep the size of the spatial ROI of the box-neighborhood (parameter $a$), unchanged for all experiments ($a=10$). An overview of the entire process thus described is shown in Algorithm \ref{alg:sparse_to_dense}.

\begin{algorithm} 
 \caption{Point match extraction (for $e_i$)}
 \label{alg:point_matching}
\begin{algorithmic} 
 \REQUIRE All spike-events $E=\{e_1,e_2,...,e_N\}$ , \   Event to be matched $e_i=(x_i,y_i,t_i)$, \   $ T_{\delta}=[\delta t_1,\delta t_2,..,\delta t_k] $ \\ Overlap parameter $r$ 
 \ENSURE Event displacement parameters $\overrightarrow{\delta_i} = (\delta x_i,\delta y_i,\delta t_i)$, Matched pair of spike-count vectors $(\mathbf{pc}_i,\mathbf{pc'}_i)$ 
 \FOR {$\overrightarrow{\delta_e}  \in \{-1,1\} \times  \{-1,1\} \times T_{\delta} $}
 \STATE $\mathbf{D}(\overrightarrow{\delta_e} ) \leftarrow \ \left \lVert F_{PCA} \left ( e_i,  \frac{\delta t}{1+r} \right ) - F_{PCA} \left ( e_i +\overrightarrow{\delta_e} ,  \frac{\delta t}{1+r} \right )  \right \rVert ^2  $
\ENDFOR \\
$\overrightarrow{\delta_i} \leftarrow _{argmin \  \overrightarrow{\delta_e} } \mathbf{D}(\overrightarrow{\delta_e} )$; \\
 $(\mathbf{pc}_i,\mathbf{pc'}_i) \leftarrow \left ( F_{PCA} \left ( e_i,  \frac{\delta t}{1+r} \right ) , F_{PCA} \left ( e_i + \overrightarrow{\delta_i},  \frac{\delta t}{1+r}  \right ) \right ) $ 
;
\end{algorithmic}
\end{algorithm}

\subsection{Extraction of Feature Vector Matches} \label{sec:feature_matching}
\begin{figure}[htbp]%
    \centering
\subfigure[Search space for a fixed temporal displacement, $\Delta T$.]{\includegraphics[width=0.24\textwidth]{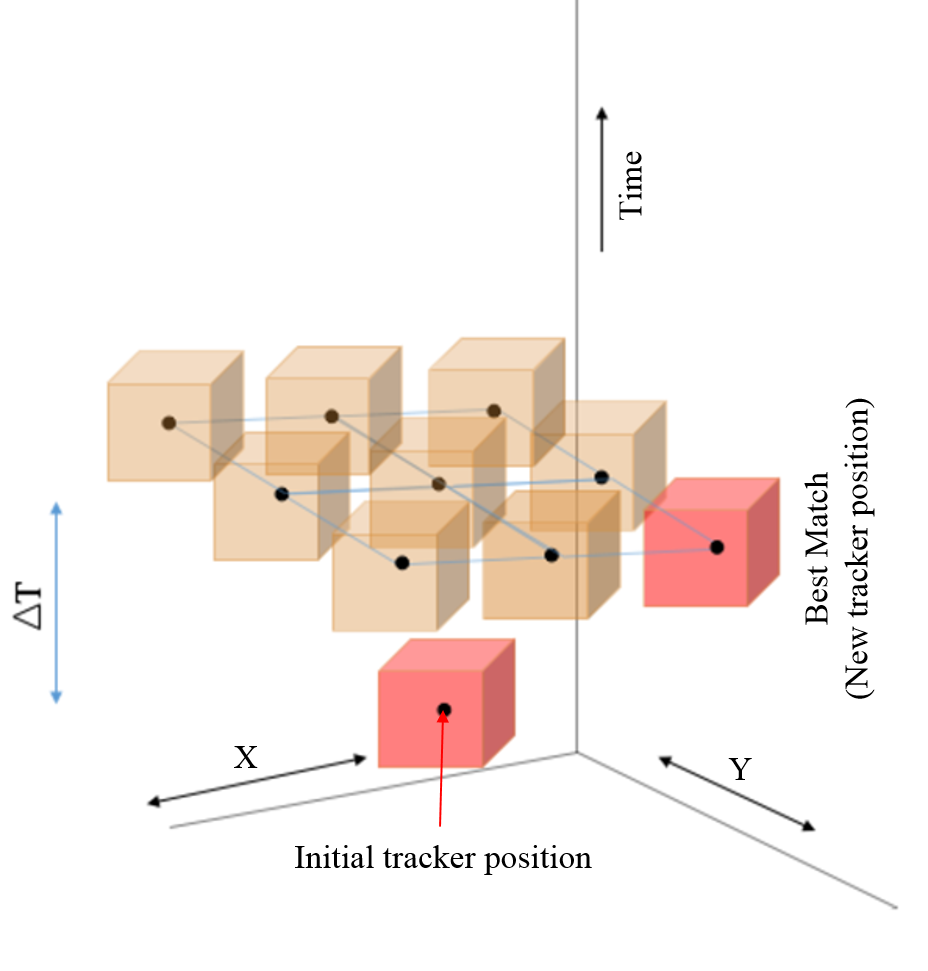}\label{fig:matching_space-a}}
\subfigure[Our proposed approach with variable temporal displacements.]{\includegraphics[width=0.24\textwidth]{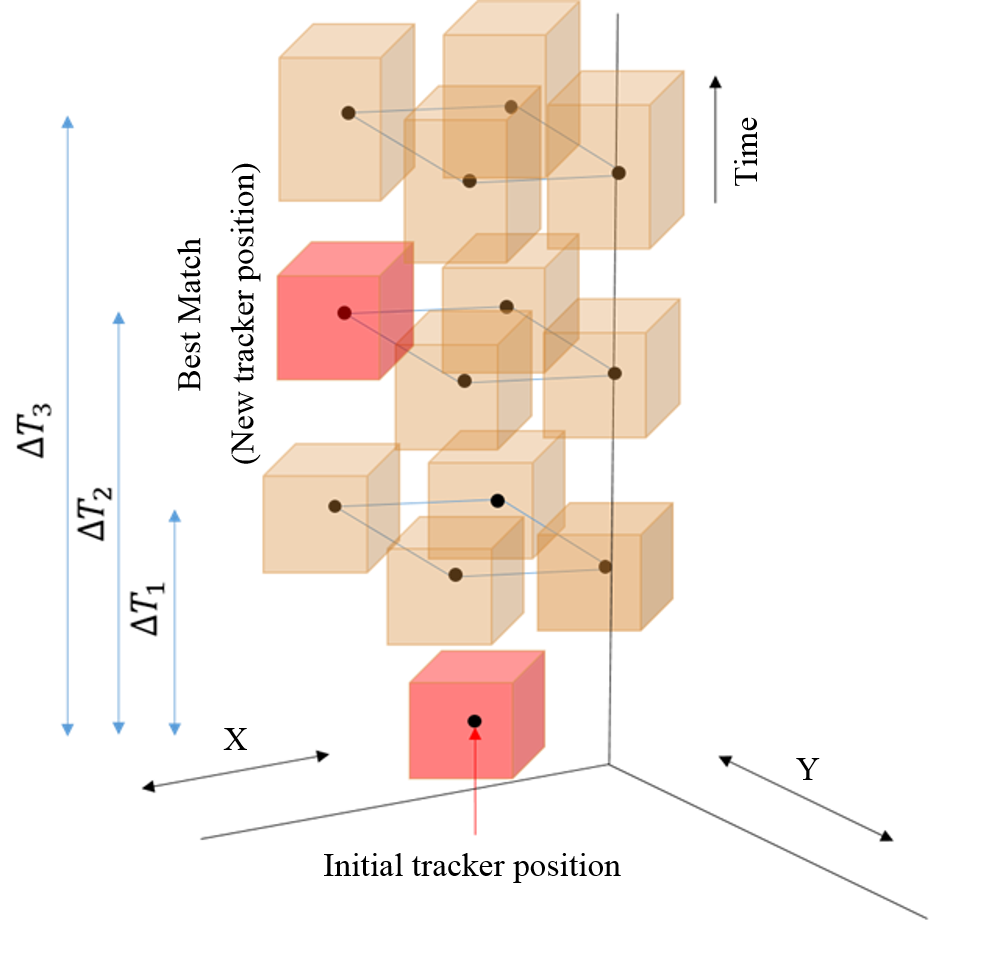}\label{fig:matching_space-b}}
\caption{Contrast between two approaches to updating tracker position is shown. In (a), lower speeds of the feature point would lead to smaller spatial displacements and vice versa. Whereas in (b), a fixed spatial displacement implies that temporal displacement controls the speed of the tracker. Note the increase in temporal dimension of box neighbourhoods as $\Delta T$ increases, corresponding to time-scaling of the pattern in response to variable speeds. For simplicity overlapping box neighborhoods are not shown.}%
\label{fig:matching_space}%
\end{figure}
For an event $e_i=(x_i,y_i,t_i)$, we seek its event displacement parameters $(\delta x_i, \delta y_i, \delta t_i)$. These provide the location of the event  $(x_i +\delta x_i, y_i + \delta y_i)$, at a later time $t_i+\delta t_i \ (\delta t_i>0)$. To estimate them, a dissimilarity measure between a pair of feature vectors extracted at different $(x,y,t)$ locations is needed. Given arbitrary displacement parameters $\overrightarrow{\delta_e} = (\delta x,\delta y,\delta t)$, we choose the Euclidean distance between the PCA feature vectors
\begin{equation}
\Big\|
 F_{PCA} \left ( e_i, T\right )  -  F_{PCA} \left ( e_i + \overrightarrow{\delta_e} , T\right )
\Big\|^2 , 
\end{equation}
as the measure of dissimilarity between the event-distribution around $(x_i,y_i,t_i)$ and $(x_i+\delta x, y_i+\delta y,t_i+\delta t)$.  

The search space of probable $(\delta x,\delta y,\delta t)$ must permit variable speeds of the feature. One approach to finding these parameters involves doing a local grid search over possible values of $(\delta x, \delta y)$ as shown in Fig. \ref{fig:matching_space} (a), with $\delta t$ kept fixed. Instead, here we assign $|\delta x|=1$ and $|\delta y|=1$ and vary $\delta t$ to account for different speeds. Let 
\begin{equation*} \label{eq:time_dims}
T_{\delta}=\{\delta t_1,\delta t_2,..,\delta t_k\}
\end{equation*}
be the set of possible temporal displacements. In other words, these are the possible times taken for a visual feature to have moved across the spatial domain by 1 pixel in either the $x$ or $y$ direction (or both). The set of possible event displacements is represented by the cartesian product
\begin{equation}
\Delta_E=\{-1,1\} \times  \{-1,1\} \times T_{\delta}.
\end{equation}
The optimum displacement parameters $(\delta x_i,\delta y_i,\delta t_i)$ are the $(\delta x,\delta y,\delta t) \in \Delta_E$ for which the Euclidean distance
\begin{equation}
\Big\|
 F_{PCA} \left ( e_i,  \frac{T}{1+r} \right )  -  F_{PCA} \left ( e_i + \overrightarrow{\delta_e} , \frac{T}{1+r}\right )
\Big\|^2, 
\end{equation}
is minimum. Note that the temporal dimension of the box neighborhood changes proportionally to the temporal displacement $\delta t$ as $\Delta T=\frac{\delta t}{1+r}$. This proportional \textit{stretching} of the box-neighborhood along the time dimension can be explained by speed considerations of the feature (more details in section \ref{sec:discussion}). The parameter $r \in (0,1)$ controls the overlap between the two box neighbourhoods. 
 
The search space of possible event displacement parameters is shown in Fig.~\ref{fig:matching_space} (b). For the sake of simplicity, non-overlapping box neighborhoods are shown. Formally the point matching algorithm is summarized in Algorithm ~\ref{alg:point_matching}. The algorithm outputs a matching pair of feature vectors $(\mathbf{pc}_i,\mathbf{pc'}_i)$, where $\mathbf{pc'}_i$ is the PCA based feature representation corresponding to the location of the best match for the $i^{th}$ event. To extract all matching pairs across the training data, Algorithm ~\ref{alg:point_matching} is repeated over all events. The quality of these feature matches is clearly dependent on the PCA features. 
 
The matching pairs of features $\mathbf{pc}_i$ and $\mathbf{pc'}_i$ are essentially the vectorized form of the spike-count matrices extracted at locations $(x_i,y_i,t_i)$ and $(x_i + \delta x_i , y_i + \delta y_i, t_i + \delta t_i)$. These pairs of vectors extracted throughout available event-data are used to further enhance the feature representations. This is described in the following section.

\begin{figure*}[htbp]%
    \centering
\includegraphics[width=0.8\textwidth]{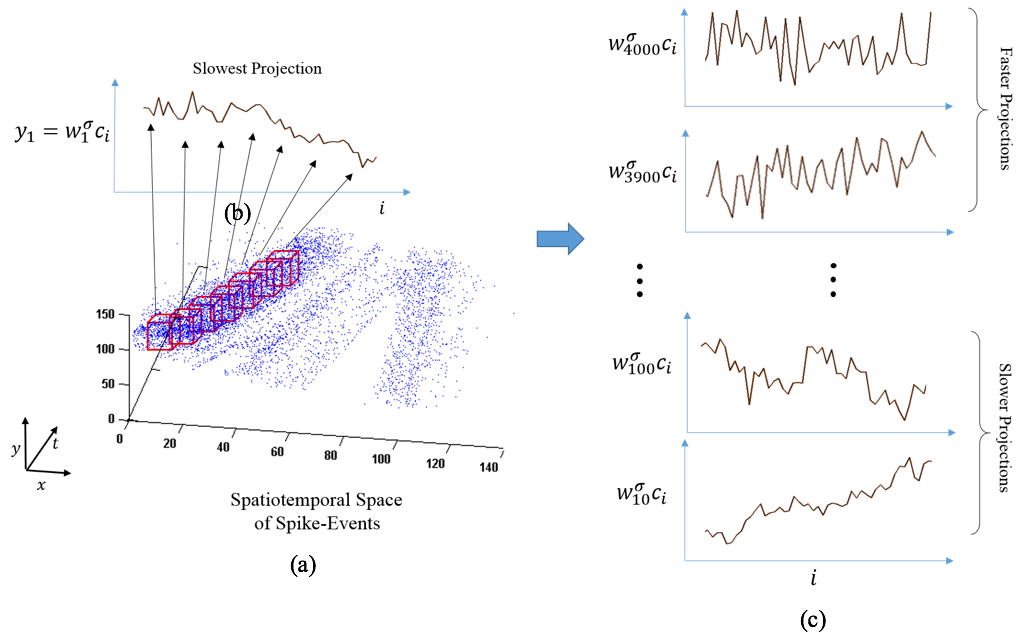}
\caption{An example of the slow feature extraction process on spike-event data (blue dots) obtained from the traffic dataset. (a) the boxes highlight the spatio-temporal path of the tracker, which in this case was made to follow a corner point. $c_i$ represents the vectorized spike-count matrix obtained from the spatiotemporal bounding boxes (in red) of size $10\times10\times200$ms, (b) the sequence of values of the projection of $w_1^\sigma$ is shown on the successive boxes formed by the tracker trajectory, numbered by the variable $i$, (c) shows contrast between sequences formed by weight vectors having different values of slowness. Note that the weight vectors $w_k^\sigma$, $k=1$ to $4000$, are ordered with respect to decreasing values of the slowness parameter. }%
\label{fig:sfa_progression_new}%
\end{figure*}
\subsection{Learning slow features from matches}


Once all the matched pairs of spike-count vectors $(\mathbf{pc}_i,\mathbf{pc'}_i)_{i=1}^{K}$ have been extracted from the training data, we wish to estimate linear projections which use this correspondence information to create robust features. The goal here is to learn features which show small variation within a match, but large variation across different matches. Note that each matching pair of spike-count vectors is indicative how the spatiotemporal activity changed when a feature re-appeared. Therefore by minimizing variation of the feature within a match, we can make the feature representations robust to the visual transformations that occur within the short time between the two sightings of the feature. The process of learning these features are detailed as follows. 

The 3D matrix forms of $\mathbf{pc}_i$ and $\mathbf{pc'}_i$ are smoothed with the Gaussian kernel $\mathcal{N}(0,\Sigma)$, to generate modified feature matches $(\mathbf{pc}^{\Sigma}_i,\mathbf{pc'}^{\Sigma}_i)_{i=1}^{K}$. This is the same convolution step used before in section \ref{sec:pca_reduce}. We wish to find linear features, computable through weight vectors $\mathbf{W}_{SFA}$, which adhere to the previously mentioned constraints. To achieve that, we use a method already present in literature, called slow feature analysis (SFA). First, we highlight a brief summary of the methods used in SFA. 

Given a temporal sequence of $d$ dimensional vectors 
\begin{equation*}
\left(\mathbf{X}_1,\mathbf{X}_2,...,\mathbf{X}_i,..,\mathbf{X}_l\right),
\end{equation*}
SFA defines a \textit{slowness} parameter
\begin{equation} \label{eq:slowness_weights}
S(\mathbf{w})=\frac{\mathbb{E}_i\left[\mathbf{w}^T\mathbf{X}_{i+1} - \mathbf{w}^{\mathsmaller T}\mathbf{X}_i \right]^2}{\mathbb{E}_i\left[\mathbf{w}^{\mathsmaller T}\mathbf{X}_i-\mathbb{E}\left[\mathbf{w}^{\mathsmaller T}\mathbf{X} \right ] \right]^2}
\end{equation}
for a linear projection vector $\mathbf{w}$. Here $\mathbb{E}_i$ is the expectation (mean) operator over the index $i$. Similar to PCA, the SFA method returns an orthogonal set of $d$ linear projections $\left[ \mathbf{w}_1 \ \mathbf{w}_2 \ ... \ \mathbf{w}_d \right]$ such that their slowness values $\left( S(\mathbf{w}_1),S(\mathbf{w}_2),..,S(\mathbf{w}_d)\right)$ are in ascending order, with $\mathbf{w}_1$ having the smallest value among all possible projections. 

For our case, we find weights which minimize 
\begin{equation} \label{eq:feature_SFA}
S(\mathbf{w})=\frac{\mathbb{E}_i\left[ \left ( \mathbf{w}^{\mathsmaller T}\mathbf{pc'}^{{\Sigma}}_i - \mathbf{w}^{\mathsmaller T}\mathbf{pc}^{{\Sigma}}_i \right )^2 \right ]}{\mathbb{E}_i \left [\mathbf{w}^{\mathsmaller T}\mathbf{pc}^{{\Sigma}}_i-\mathbb{E}_{k}\left[\mathbf{w}^{\mathsmaller T}\mathbf{pc}^{{\Sigma}}_k \right ] \right ]^2}  .
\end{equation}

Observe the similarities between the above formulation and SFA. The only difference is that instead of a sequence of observations, we have pairs of them. The numerator in the equation \ref{eq:feature_SFA} is the difference in the feature value within a matching pair. The denominator is the variance of the feature value across all feature matches. We want projections which have small intra-match differences and large inter-match differences. This would mean smaller $S(\mathbf{w})$ is preferred over larger. Thus we choose only $n_{\mathsmaller{SFA}}<<d$ projections having the lowest values of $S(\mathbf{w})$. The $n_{\mathsmaller{SFA}}$ linear projections are portrayed as the columns of the matrix 
\begin{equation*}
\mathbf{W}_{\mathsmaller{SFA}} =\left[ \mathbf{w}_1 \ \mathbf{w}_2 \ ... \ \mathbf{w}_{n_{\mathsmaller{SFA}}} \right],
\end{equation*}
where $\mathbf{W}_{\mathsmaller{SFA}} \in \mathbb{R}^{(d \times n_{\mathsmaller{SFA}})}$. Applying theorem \ref{thm:circconv} as before, we generate the Gaussian smoothed equivalent of those weights as follows. 
\begin{equation}
\mathbf{W}_{\mathsmaller{SFA}}^{\mathsmaller T}\mathbf{pc}^{\Sigma}_i=(\mathbf{W}_{\mathsmaller{SFA}}^{\Sigma})^{\mathsmaller T}\mathbf{pc}_i.
\end{equation}
 The above result allows for a one-time convolution operation on the 3D matricized versions of each weight vector $\mathbf{w}_i$ to generate the permanently modified weights $\mathbf{W}^{\Sigma}_{\mathsmaller{SFA}}$, instead of performing smoothing on the spike-count matrix each time. Fig.~\ref{fig:sfa_progression_new} shows an example of the projections obtained from the trained weights changing with time. We reiterate that the only purpose of finding a PCA based feature representation is to obtain matches across the event-data which are subsequently utilized to learn more robust projections in $\mathbf{W}^{\Sigma}_{\mathsmaller{SFA}}$. For each event $e_i$, we denote the final SFA based feature extraction process by the function
 \begin{equation}\label{eq:final_features}
 F_{SFA} \left (e_i, T,n_{SFA}\right).
 \end{equation} 
The above function formulation is similar to the earlier PCA feature estimation function in equation \ref{eq:pca_features}, with an added parameter $n_{SFA}$ which controls the number of projections, and therefore the dimensionality of the feature itself. For simplicity, we omit this argument from any subsequent mentions of this function. 
 

\subsection{Evaluation: Point tracking with slow features}
%

 \textit{Trackers} are initialized at feature points throughout the event-data. Each tracker updates its position in time according to the motion of the feature point, and maintains a feature vector representation at all times. The objective is to accurately track feature points as they gradually undergo smooth transformations. Our intention here is to verify the robustness of the estimated features, and therefore we make the problem as hard as possible, by only allowing single-pixel updates on tracker position each time. The tracking algorithm follows the same framework described in section \ref{sec:feature_matching}, where feature matches were found by estimating the event displacement parameters. The only difference here is the use of $F_{SFA} \left (e_i, T\right)$  instead of the PCA features $F_{PCA} \left (e_i, T\right)$ for tracking. 

The method is elaborated in Algorithm~\ref{alg:point_tracking}. As illustrated in an example presented in Fig. \ref{fig:sfa_progression_new} (a), the tracker simply updates its spatial and temporal position based on the best matches obtained at each iteration with a reference feature vector. The reference feature vector is extracted at the initial location of the tracker  $(x_0,y_0,t_0)$. 

It undergoes spatial displacement only when the number of spike-events in the box neighbourhood of the best matching box is greater than a threshold number. This we refer to as the \textit{stopping criterion}. The stopping criterion is essential for a tracker to stay put when the visual feature does not move, and therefore does not produce enough spike-events. A similar step was proposed in \cite{multikernel_tracking}, where the \textit{activity} of each tracker was used to control whether the tracker updates its position, or stays put. 
 
\begin{algorithm} 
 \caption{Feature point tracking}
  \label{alg:point_tracking}
\begin{algorithmic}
 \REQUIRE{All Events $E=\{e_1,e_2,...,e_N\}$ \\ Initial tracker position: $\mathbf{p}_0$ \\ $ T_{\delta}=\{\delta t_1,\delta t_2,..,\delta t_k\} $ \\ Overlap parameter: $r$ \\ Spatial dimension of neighborhood: $a$ \\ Minimum event threshold for tracker displacement: $N_0$}
 \ENSURE{Tracker positions over $k$ iterations $\{ \mathbf{p}_0, \mathbf{p}_1, .. , \mathbf{p}_{k-1}) \}$}
\FOR {$i \in \{0,1,..,k-1\}$}
 \FOR {$\boldsymbol{\delta} \in \{-1,1\} \times  \{-1,1\} \times T_{\delta}$} 
 \STATE $\mathbf{D}(\boldsymbol{\delta}) \leftarrow \ \left \lVert  F_{SFA} \left (\mathbf{p}_i, \frac{T}{1+r}  \right) -  F_{SFA} \left (\mathbf{p}_i + \boldsymbol{\delta}, \frac{\delta t}{1+r}\right) \right \rVert^2  $
\STATE  $Count(\boldsymbol{\delta}) \leftarrow \left| \mathcal{B}_E(\mathbf{p}_i + \boldsymbol{\delta}, a, a, \frac{T}{1+r}) \right|$
\ENDFOR
\ENDFOR
\STATE $\boldsymbol{\delta}_{min} \leftarrow _{argmin \  \boldsymbol{\delta}} \mathbf{D}(\boldsymbol{\delta})$   
\STATE $minCount = \underset{\boldsymbol{\delta}}{min} \{ Count(\boldsymbol{\delta}) \}$ 
\IF{$minCount > N_0$}
\STATE $\mathbf{p}_{i+1} \leftarrow \mathbf{p}_i + \boldsymbol{\delta}_{min}$ \\
\ELSE
\STATE $\boldsymbol{\delta}_{min} = (0,0,\delta t_{min})$ \\ 
\STATE $\mathbf{p}_{i+1} \leftarrow \mathbf{p}_i + \boldsymbol{\delta}_{min}$ \\
\ENDIF
\end{algorithmic}
\end{algorithm}

\begin{figure}[h]
\hspace*{-1.2cm} 
\includegraphics[width=0.6\textwidth]{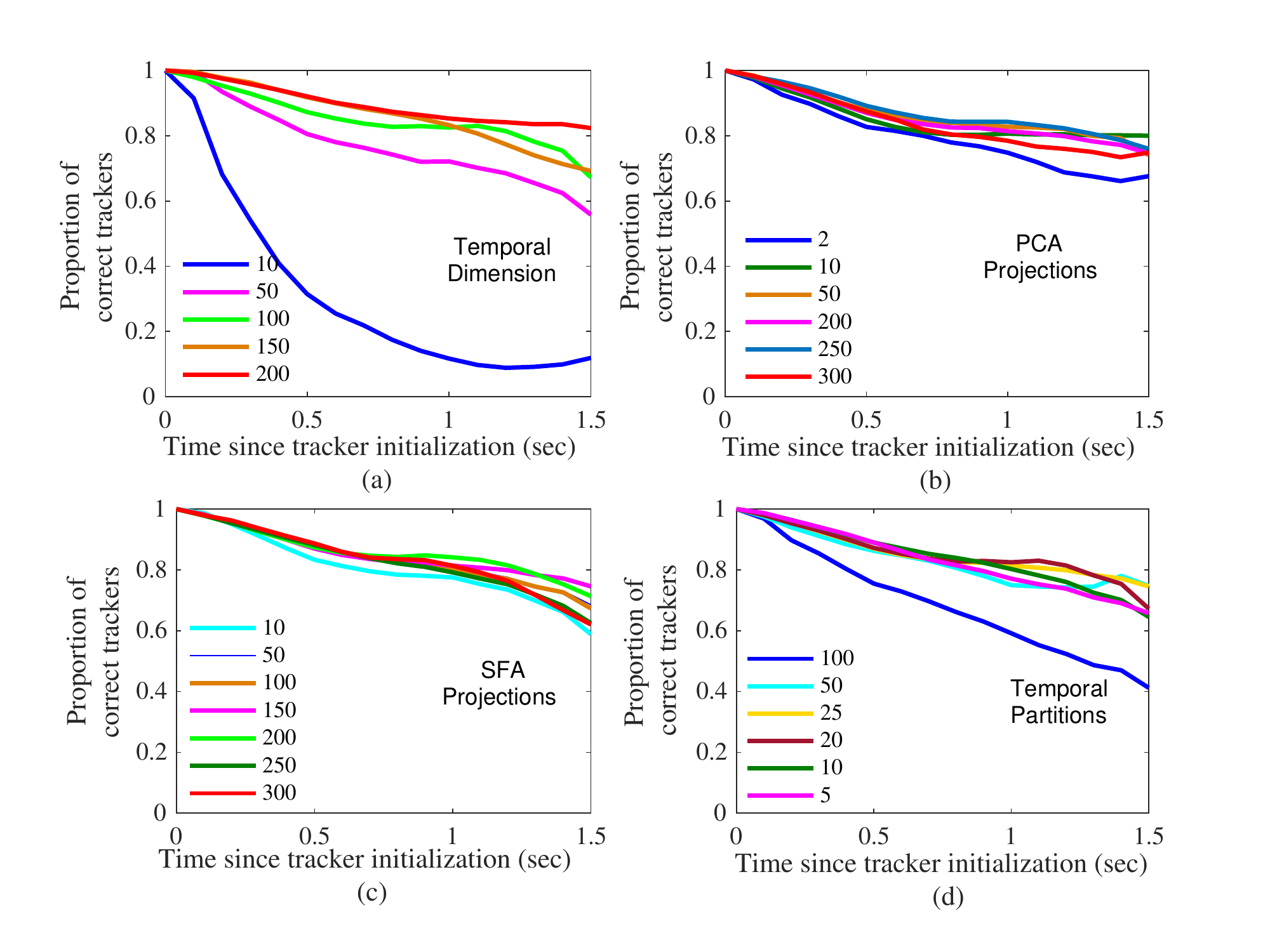}
\caption{Shown are the accuracy plots generated from testing the point tracking algorithm on the Traffic dataset. Figures (a)-(d) show all accuracy plots containing the analyses of tracker performance in the traffic dataset. In short, (a) (\textit{temporal dimension}) highlights that acquiring more information about the event pattern around a feature point, by increasing temporal ROI size, leads to better performance; (b)  (\textit{Number of PCA projections $n_{PCA}$}) indicates that using too few or too many principal components for the initial tracker matches affects performance negatively; (c)  (\textit{Number of SFA projections $n_{SFA}$}) shows that using too few or too many SFA projections affects performance negatively, with the optimum performance at around $n_{SFA}=150$; (d)  (\textit{Number of temporal partitions $M$}) indicates that preserving more time information of each event leads to better performance, but up to a point, after which performance degrades noticeably.}
\label{fig:all_accs_traffic}
\end{figure}

\begin{figure*}[h]%
    \centering
\includegraphics[width=0.8\textwidth]{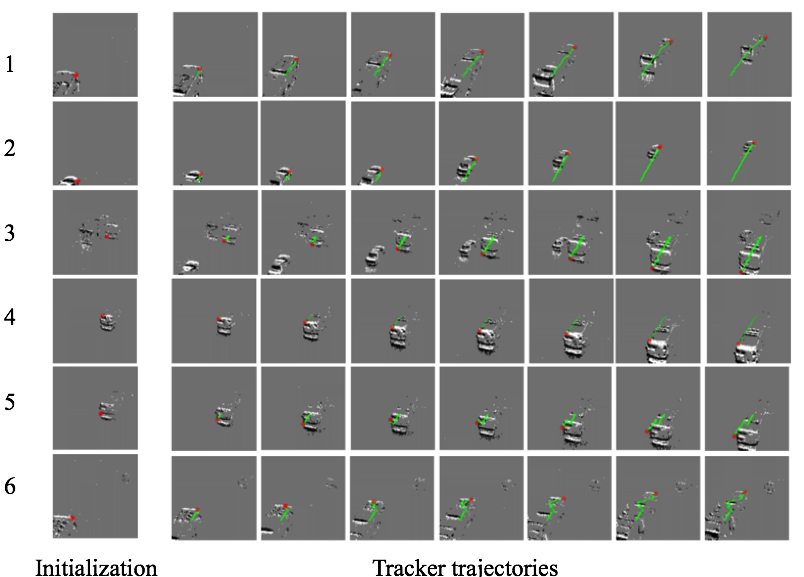}
\caption{Examples of tracker trajectories computed by our algorithm ($M=40$, $\Delta T=200 ms$, $a=10$) on the testing data. On a given row, the left most frame shows the initial location of the tracker (in red). For the images to the right, the past trajectory of the tracker is overlayed onto each image in green. }%
\label{fig:abcd}%
\end{figure*}


  
\section{Experiments and Results} \label{sec:results}
\subsection{Pre-processing}
A noise filtering routine \cite{noise_dvs} was applied to the obtained spike-events from each recording. The filtering process eliminates spike events which only have a single event in the box neighbourhood $\mathcal{BN}_E(e_i, 2,2, 3\times 10^4 \mu s)$. 

\subsection{Tracker performance analysis}

To test tracker performance, multiple ground truth trajectories of feature points were annotated across event-data. The feature points which were chosen for annotation were mainly corner-like points. We denote each $i^{th}$ ground truth trajectory by the set $\left \{x_g^i(t),y_g^i(t),t \ \middle|  0<t<T_{max}^i \right \}$, where $t$ is the time since each tracker was initialized. Similarly, the $i^{th}$ estimated trajectory from our method is represented by the set $\left \{x_e^i(t),y_e^i(t),t \ \middle| \   0<t<T_{max}^i \right \}$. The tracking performance is quantified using the following metric.
\begin{equation}\label{eq:eval_tracker}
 \tau(t)=\frac{ 
 \underset{i}{\#} \left \{ \sqrt{(x_{\small{g}}^i(t)-x_{\small{e}}^i(t))^2 + (y_{\small{g}}^i(t)-y_{\small{e}}^i(t))^2} \leq 7 \right \} 
}{\underset{i}{\#} \left\{T_{max}^i>t\right\}}  
\end{equation}

Here ${\#}$ is the set cardinality operator. Note that $\tau(t)$ represents the proportion of trackers which are within a distance of seven pixels from their corresponding ground truth locations at time $t$ since their initialization. The plot of $\tau(t)$ over a fixed time interval of $0$ (initialization) to $1.5$ seconds is used as a quantifier of tracking performance for all our experiments. We refer to this curve as the \textit{accuracy plot}. Most of the analyses presented in the following sections involves monitoring accuracy plots and their response to changes in algorithm parameters.

\subsection{Traffic (translation and scaling)}
The DVS was placed on the handrail of a pedestrian overpass facing the road below. The event data generated from vehicular motion in the direction towards and away from the DVS was recorded. The recording was done over a period of 15 minutes in broad daylight. The level of traffic ranged from one to five vehicles at any time instant, with a few instances showing occlusion. The curved road resulted in non-linear tracker trajectories. Fig.\ref{fig:traffic_data} shows example data from the recording. The initial five minutes of the data was used for training and the rest (10 minutes) for testing. The annotated data consists of 130 tracker trajectories evenly spaced throughout the data\footnote{The data along with the annotations are available online at \textit{https://files.fm/u/933uurrr?k=90a4c919}.}. Fig.~\ref{fig:abcd} displays selected tracker trajectories estimated by our algorithm on the test annotations. In the cases numbered from 1 to 4, the tracker trajectory is accurate throughout the recording. However, for cases 5 and 6, note that the tracker strays away from the feature point, but is eventually able to recover and regain accurate positioning. We observed such cases where the tracker deviates from the correct feature position due to less spike-events recorded. However, in most cases, subsequent recovery occurred, indicating the robustness of our method. Note the extent of the change of scale, during the entire duration a feature point is visible in the camera frame.

%


\subsubsection{Varying temporal dimension}

The temporal dimension of the box neighbourhoods, $T$, was varied, and a new set of weights $\mathbf{W}_{PCA}$ and $\mathbf{W}_{SFA}$ were learnt each time.  As Fig.~\ref{fig:all_accs_traffic} (a) shows, the performance is usually worse for lower values of $T$. Increasing $T$ improves the performance of the tracker, as $T =200$ ms returns the highest proportion of correct trackers on average. However, we consider 200 ms to be a large temporal dimension for extraction of spatiotemporal features, and therefore for the subsequent analyses, we set $T=100$ instead.

\subsubsection{Varying the number of PCA projections}
The number of projections $n_{\mathsmaller{PCA}}$ of $\mathbf{W}_{\mathsmaller{PCA}}$ used to extract matches in Algorithm~\ref{alg:point_matching} was varied, and the subsequent accuracy plots obtained with the learned SFA weights were compared.  For this analysis, parameter values chosen were $\Delta T =100 \ ms$ and $M=25$. Fig.~\ref{fig:all_accs_traffic} (b) contains the relevant accuracy plots. Overall, we find that choosing only 2 principal components to obtain matches ($n_{\mathsmaller{PCA}}=2$) results in poor tracking performance. For larger values of $n_{\mathsmaller{PCA}}$ the performance of the tracker improves. However, no significant difference in performance can be observed between $n_{\mathsmaller{PCA}} = 10$ and $n_{\mathsmaller{PCA}}=250$. Beyond that however, we find that performance drops, indicating that using too many PCA components for generating the tracked pairs of feature matches is not advisable. 

%

\subsubsection{Varying the number of SFA projections}
Another free parameter in our algorithm is $n_{\mathsmaller{SFA}}$, the number of slow projections used in Algorithm~\ref{alg:point_tracking}. We compare the accuracy plots with different values of $n_{\mathsmaller{SFA}}$ in Fig. \ref{fig:all_accs_traffic} (c).  Note that a small value of $n_{\mathsmaller{SFA}}=10$ expectedly yields low accuracy ($60\%$ at $t=1.5s$).  A relatively larger value of $n_{\mathsmaller{SFA}}=150$ yields stable performance. Still larger values of $n_{\mathsmaller{SFA}}$ leads to a decrease in accuracy to almost as low as that obtained for $n_{\mathsmaller{SFA}}=10$. As the projections are arranged in ascending order of their slowness, the subsequent projections are noisier and therefore can worsen tracking accuracies.

\subsubsection{Varying temporal partitions}

The number of temporal partitions $M$ was varied. The temporal dimension of the box neighbourhoods was kept unchanged at  $T=100 \ ms$. The accuracy plots were studied for different values of $M$, and are shown in  Fig.~\ref{fig:all_accs_traffic} (d). The lowest number of partitions ($M=5$, 50 FPS) gives a steadily decreasing performance graph. In contrast, higher values around $M=25$ (250 FPS) demonstrate better average performance. A very high number of temporal partitions ($M=100$, 1000 FPS) makes performance suffer noticeably.


\subsubsection{Comparisons with other methods} \label{sec:comparison_traffic}
\begin{figure}[h]
\centering
\includegraphics[height=5cm,keepaspectratio]{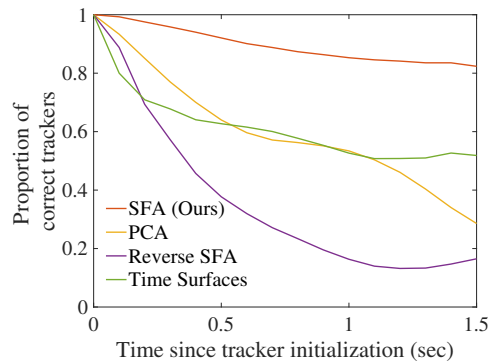}
 \caption{Accuracy plots comparing point tracking performance of our method with other feature estimation strategies.}
 \label{fig:comparisons_babel}
 \end{figure}
We compare our best performing method ($\Delta T=200ms$) with the following ways of generating spatiotemporal features:
 \begin{figure*}[h]%
    \centering
\includegraphics[width=1\textwidth]{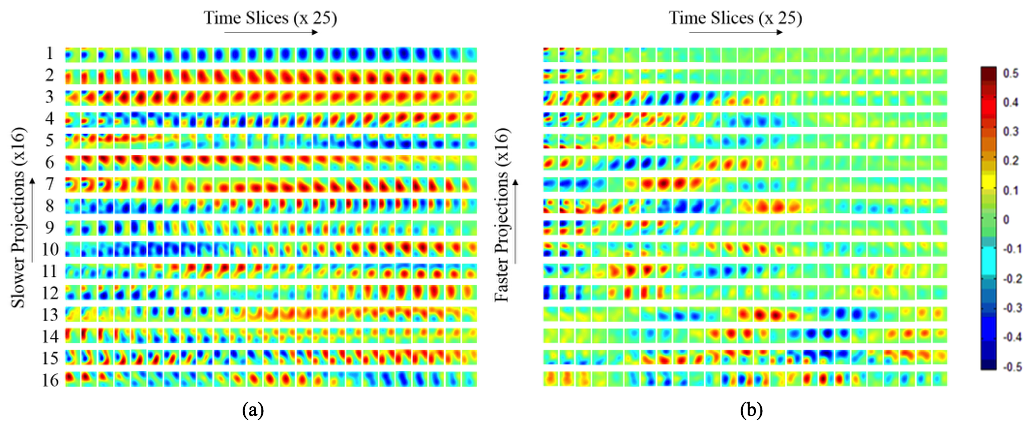}
\caption{Spatiotemporal weights generated by our algorithm: (a) 16 spatiotemporal weight vectors which have the most slowly changing projections, from the weight matrix $W^{\boldsymbol{\sigma}}_{\mathsmaller{SFA}}$ . (b) shows 16 weight vectors with most rapidly changing projections. On a given row, each image is obtained by keeping the temporal index fixed of the 3D weight matrices. Each image represents a time slice of the 3d weight matrix of size $10\times10$. A total of 25 time slices are shown for each weight matrix. The color coding of the pixel values is shown to the right. The images are ordered from left to right w.r.t increasing temporal index.}%
\label{fig:weights_3d}%
\end{figure*}
\begin{itemize}
\item PCA: This is an obvious choice of projections other than our SFA based projections, as it only considers dimensionality reduction. Instead of $\mathbf{W}_{\mathsmaller{SFA}}$, $\mathbf{W}_{\mathsmaller{PCA}}$ was used as the weight matrix in Algorithm~\ref{alg:point_tracking}. The number of projections $n_{\mathsmaller{PCA}}$ was fixed at 150.
\item Reverse SFA: This comparison serves to measure the contrast in performance of the slow weights as opposed to the faster projections. To do this, instead of using the $n_{\mathsmaller{SFA}}=150$ slowest projections to construct $\mathbf{W}_{\mathsmaller{SFA}}$, we take the 150 projections having the highest value of slowness parameter $S(w)$. 
\item Time surfaces: Time surfaces represented as $T_s(x,y,t)$ were used as the feature vectors. All the parameters and steps involved in Algorithm~\ref{alg:point_tracking} were unchanged, except $\mathbf{D}(\delta x, \delta y,\delta t)$ which was changed to $\left(e^{-(t_0-T_s(x_0,y_0,t_0))/ \tau}-e^{-(t_i-T_s(x_i+\delta x,y_i+\delta y,t_i + \delta t))/ \tau} \right )^2$. The parameter $\tau$ was set at $50ms$.  
\end{itemize}
The accuracy plots obtained from all the above methods is shown in Fig.~\ref{fig:comparisons_babel}. SFA based projections obtained from our method achieve better performance than the others tested here. The weights from reverse SFA perform very poorly, as expected. Moreover, time-surfaces and PCA features show very similar performance until $t =1$s, after which the performance of PCA features degrade considerably in relation to time-surfaces.  

\subsubsection{Visualization of the SFA weight matrix} \label{sec:sfa_weights}

The smoothed spatiotemporal weight vectors $\mathbf{W}^{\boldsymbol{\sigma}}_{\mathsmaller{SFA}}$ obtained from the traffic data with parameter values $M=25$ and $\Delta T =100ms$ are shown in Fig.~\ref{fig:weights_3d}. In this example, only the sixteen projections with smallest values of slowness parameter $S(\mathbf{w})$ are shown in Fig.~\ref{fig:weights_3d} (a). Similarly, the sixteen projections with the largest values of their slowness parameters are shown in Fig.~\ref{fig:weights_3d} (b). Each 3D matricized weight vector from the columns of $\mathbf{W}^{\boldsymbol{\sigma}}_{\mathsmaller{SFA}}$ is of size $(10\times10\times25)$, shown in the form of twenty five $(10\times10)$ images. The pixels are color coded with $dark \ blue$ being the most negative weights and $dark \  red$ the most positive. The main observations are summarized below. 
\begin{itemize}
\item The weight vectors with rapidly changing projections exhibit high degree of sparseness compared to the slowly changing projections. As such, the slowly changing and rapidly changing projection weights can be clearly differentiated.  
\item The slowest changing projection weights are shaped as big 3D blobs which count the number of events inside them. Some of them have patterns which translate with time, indicating responsiveness to translation. 
\item In Fig.~\ref{fig:weights_3d} (a), most weights show smoothly shifting patterns across time. Some of them have corner-like shapes (4, 7 and 13). A few weights respond to event rate changes (10 and 12), indicating that they encode whether the feature point is moving towards or away from the camera (corresponding to vehicles moving in either direction). There are also weights which respond to specifically oriented edges moving in a certain direction (6 and 8). These properties indicate similarities of these weights to visual cortical receptive fields of different kinds (simple, complex and hypercomplex \cite{hubel_godly}). 
\end{itemize}

\subsection{Grid Pattern Rotation (translation and rotation)}
\begin{figure*}[h]
\centering
\includegraphics[width=0.7\textwidth]{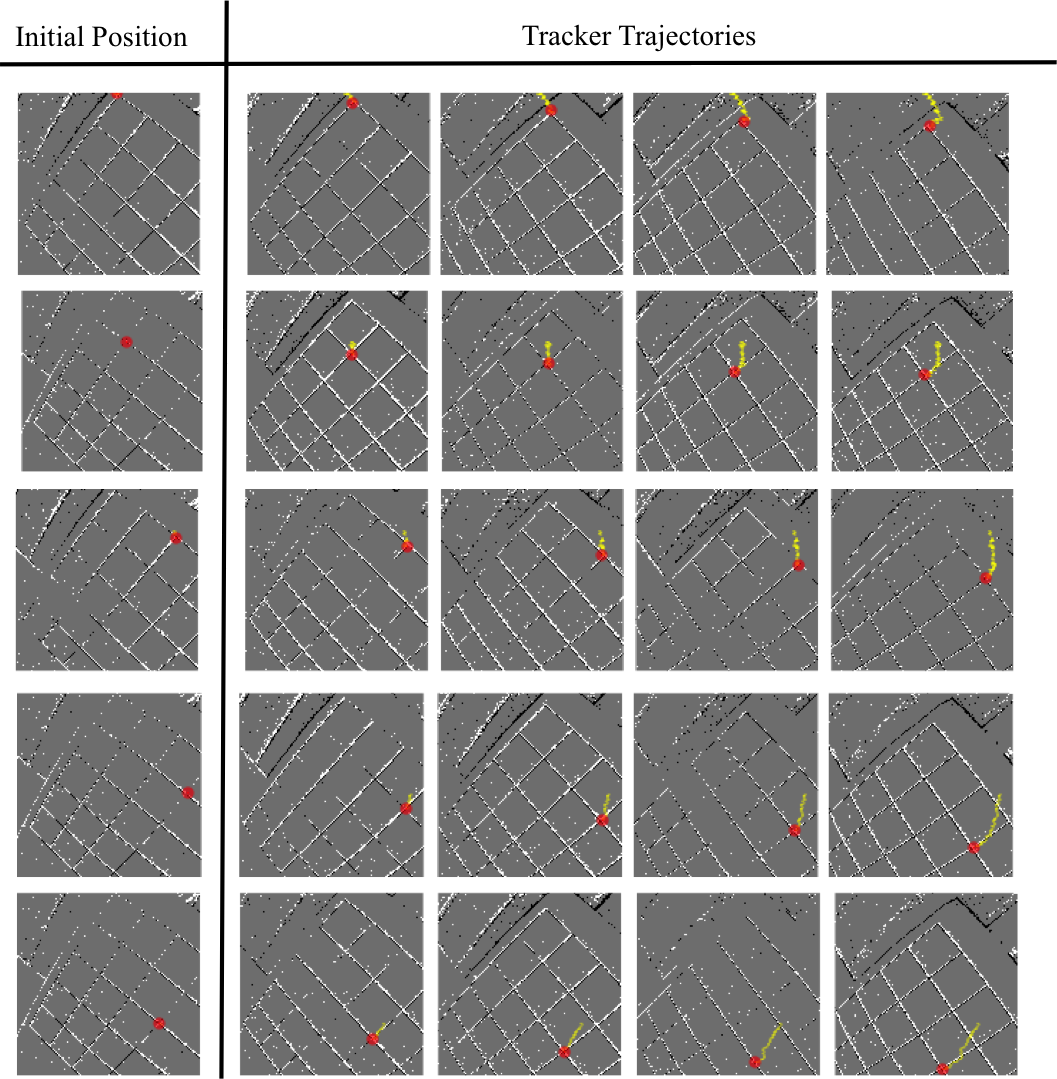}
\caption{Tracker paths estimated by the algorithm for a grid pattern rotation data.}
\label{fig:rotation_results}
\end{figure*}
\begin{figure}[h]
\centering
\includegraphics[width=0.48\textwidth]{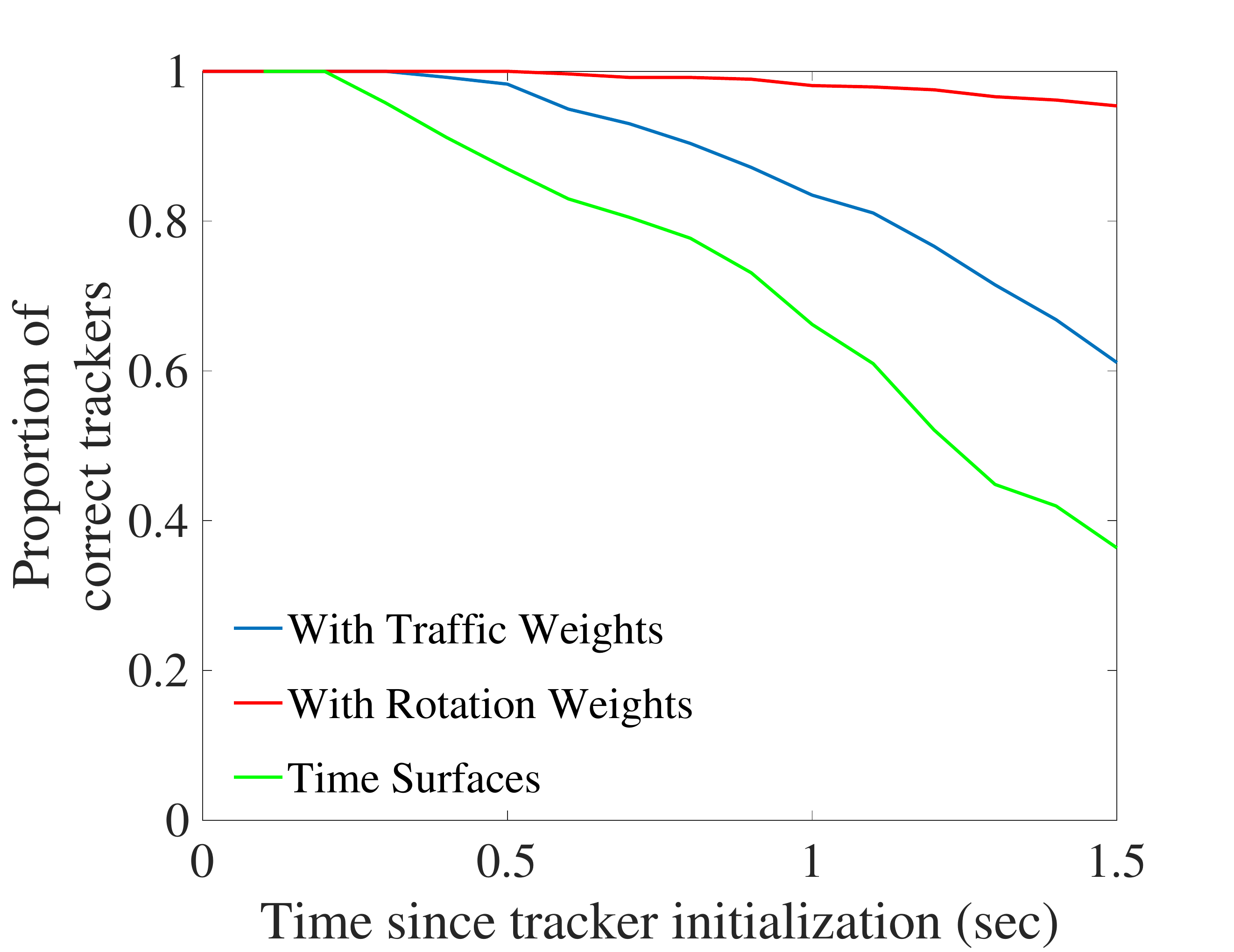}
\caption{Comparison of the accuracy plots obtained with our spatiotemporal weights learnt from the rotation data itself, the spatiotemporal weights learnt from traffic data and time-surface based features.}
\label{fig:rotation_results_comparison}
\end{figure}

A $5 \times 5$ square grid of length (this many cms) was statically placed in front of a handheld DVS camera. The camera was simultaneously rotated and translated in a near-smooth motion from left to right, at all times facing the grid pattern. Spatiotemporal weights were learnt from scratch with this new data, from which spatiotemporal features were derived. The weight matrices were of size $(10 \times 10 \times 50)$, spread over $100$ ms, and therefore each individual voxel was of size $(1 \times 1 \times 2 \ ms$). 
The objective of this experiment was to demonstrate that complex rotation invariant features can also be learnt by the system, if trained on appropriate data. The hand-held camera rotated for approximately 30 degrees during that time. We found that 30 degrees of rotation was enough to discriminate the performance between rotation-invariant features and earlier translation/scale invariant features from the traffic data. \footnote{Note that the objective here is not to learn features which show complete 360 degree rotational stability to visual patterns, but to more realistic rotational distortions induced by hand-held motion.} The corner points and the intersections on the grid were chosen for the annotated trajectories (a total of 22 trajectories).

We analyze the tracker performance (Fig. \ref{fig:rotation_results}), and compare it with the following approaches:
\begin{itemize}
\item Spatiotemporal weights learnt from traffic data: The spatiotemporal weights, $\mathbf{W}_{SFA}$, obtained from the traffic data, were used for point tracking instead of the weights learned from the grid rotation data itself.
\item Time Surfaces: As explained in section \ref{sec:comparison_traffic}, time surface based feature representations were used for point tracking, keeping all the other algorithm parameters fixed. 
\end{itemize}

The accuracy plots obtained from each method is shown in Figure \ref{fig:rotation_results_comparison}. Notice that the spatiotemporal weights learnt from traffic data still outperforms the time-surface based feature matcher. More importantly, the new spatiotemporal weights learnt form the grid rotation data itself is able to achieve robust performance (near 95\% over 1.5 seconds), outperforming the compared methods. This highlights that the new features are indeed able to capture rotational invariance much more than the previously learnt traffic features, demonstrating the adaptability of the features learnt to the data provided. 

    \subsection{Cross dataset tracking performance for MNIST} 


Here we analyze tracking performance when the spatial patterns are the same (roughly), but the motion patterns are different across datasets. The main goal is to observe the cross dataset performance, when weights trained from a certain motion pattern are used for tracking in the other dataset which contains a different motion pattern. 
Two sets of digits were printed from the MNIST dataset were printed, and a UR10 robot with a DVS fitted on the end effector was made to move in two different motion profiles. Hence, two sets of training and testing data were recorded with the DVS, each consisting of the digits moving in a certain motion profile. The two motion profiles used are linear motion and triangular wave motion, as demonstrated in Fig. \ref{fig:mnist_all} (b). From the training data for each motion profile let us denote the learnt spatiotemporal weight matrices as matrices $\mathbf{W}_{(linear)}$ and $\mathbf{W}_{(triangular)}$. For the tests, both sets of weights were used to validate tracking performance on each dataset (linear and triangular). The parameters used in this experiment were $n_{\mathsmaller{PCA}}=10$, $n_{\mathsmaller{SFA}}=150$, $M=25$ and $\Delta T=100ms$. 
\begin{figure}[h]%
\centering
\includegraphics[width=0.48\textwidth]{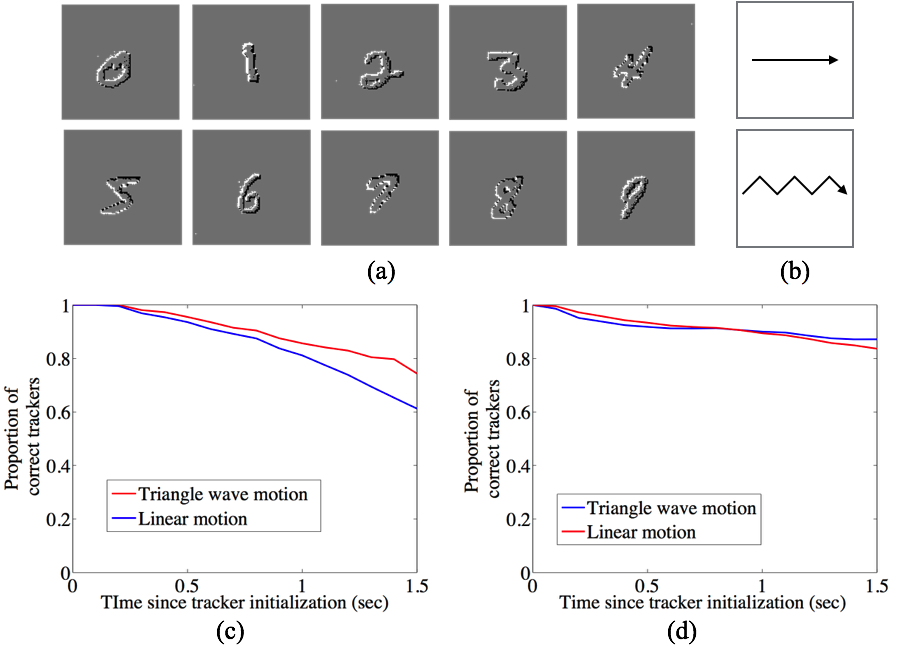}
\caption{Example motion profiles for the MNIST digits and the square corner tracking experiments. (a)  Linear motion profile and (b) triangle-wave motion profile of an example digit from the MNIST dataset is shown. The top and bottom images constitute the initial and final locations of each digit respectively, with its trajectory overlaid in green. Similarly shown in (c) are the initial and final locations of the square along with the superimposed corner trajectories.}%
\label{fig:mnist_all}
\end{figure}

The results are shown in Fig.~\ref{fig:mnist_all} (c) and Fig.~\ref{fig:mnist_all} (d). From the plots, we notice that the cross-dataset tracking performance only suffers a little, despite two different motion profiles. It is because the spatial patterns remain consistent across the two datasets. However, as expected there is a noticeable increase in tracking accuracy when the trained weights are applied to their own datasets. Additionally, we observe that the cross-dataset performance of the triangle wave motion trained weights eventually matches the linear motion weights (as shown in Fig.~\ref{fig:mnist_all} (c)) on the linear motion data itself. This can be explained as follows. Recall that the DVS only responds to moving edges, with greater spike-events generated when the edge moves along its normal direction. The triangular wave motion data clearly has a wider range of motion direction than the linear motion data. Therefore, $\mathbf{W}_{(triangular)}$ is subject to input where more edge information is revealed than its counterpart $\mathbf{W}_{(linear)}$. Naturally $\mathbf{W}_{(triangular)}$ learns to better encode the spike-event patterns than its counterpart.

\subsection{Intermittent motion}
Here we simply demonstrate the capability of the method to be robust to intermittent pauses in the motion of a feature point. For the experiment, a square shaped stimulus was translated in a piecewise linear curve. As the square moved, it was paused for $0.1$ seconds at $t=0.5s$ and $t=1s$ before changing motion direction. The stopping criterion used in Algorithm~\ref{alg:point_tracking} is pivotal to ensure that the tracker stays put as all motion ceases. Fig.~\ref{fig:stopping_test}  (a) shows the start and end positions of the square superimposed with the point trajectories. To compare, we quantify tracking performance w.r.t the average tracker displacement which is defined as follows. 
\begin{equation} 
 dist(t)=\frac{\sum_{i,T_{max}^i>t}\sqrt{(x_{\small{g}}^i(t)-x_{\small{e}}^i(t))^2 + (y_{\small{g}}^i(t)-y_{\small{e}}^i(t))^2}}{\underset{i}{\#}  \left\{T_{max}^i>t\right\}}
\end{equation}

Note that tracking performance does not suffer for the algorithm without the stopping criterion until $t=0.5s$. For $t>0.5$, the average displacement of the trackers without the stopping criterion increases rapidly, eventually reaching $20$ pixels. The tracker with the stopping criterion however maintains performance as desired. 
\begin{figure}[h]%
    \centering
\includegraphics[width=0.5\textwidth]{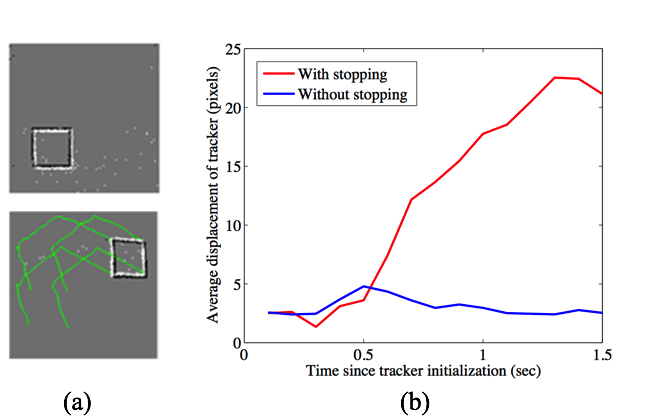}
\caption{Accuracy plots comparing Algorithm~\ref{alg:point_tracking} performance with the stopping condition and without. }%
\label{fig:stopping_test}
\end{figure}

\section{Discussion} \label{sec:discussion}

Feature extraction algorithms for frame-based input has been either general or dataset oriented (learned from the data) \cite{sparse_spatiotemporal}. A similar trend is also observed in the event-based scenario, wherein features are either defined to be general i.e. data-independent \cite{motion_feature_event,hfirst}, or tuned to the specific dataset it is trained on \cite{hots_garrick,echostate_features,invariant_thusitha}. To our knowledge, the proposed method here is the first example of an event-based feature learning algorithm which targets spatiotemporal invariance in the feature representations, fine-tuned to the spatiotemporal context of the training dataset. As evidenced by the point tracking results on traffic and rotation data (Fig. \ref{fig:rotation_results_comparison} and Fig. \ref{fig:comparisons_babel}), we find that our method is able to learn invariant spatiotemporal features fine-tuned to the dataset it has been trained on. In doing so, the method is able to obtain robust tracking performance on the respective datasets. The adaptability of the features is clearly demonstrated through Fig. \ref{fig:rotation_results_comparison}, where features learnt from grid rotation data, which contains mostly translation and rotational motion, shows considerably higher performance than the features generated from the spatiotemporal weights learned from the traffic counterpart. This adaptability essentially points to the invariant aspects of the learnt spatiotemporal features, which in our case involves translation, scale changes and rotation across the two datasets. Moreover, on the grid rotation data, features learnt from traffic are found to outperform time-surface based features, which are general. This implies that features trained on a specific dataset are not overfitted to the data, and can generalize better than time-surface based features, which are data-independent. 

Spike-event data is inherently sparse, including the spike event count matrix. However, since such a sparse representation of the spike-event pattern is not robust to slight changes in the spike-event locations, we add a gaussian smoothing step. Theorem 1 shows us that instead of smoothing the spke-event count matrix, we can smooth the spatiotemporal weights. There are two implications of this important result. First, this allows the method to have a completely event-based asynchronous implementation in practice. This is because, the new set of smoothed spatiotemporal weights can be projected on the sparse spike event count matrix, requiring multiplications only at voxels containing a non-zero number of spike-events. This in turn indicates that the output of a spatiotemporal projection can be updated with each incoming spike-event. The second implication of that theorem, is that the resulting smoothed spatiotemporal projections will show inherent spatiotemporal smoothness in their values. Note that this is achieved without enforcing any other conditions on the optimization function, such as a smoothness constraint on the weights. 

The sparse, high-dimensional spike-event data is essentially mapped onto a dense, low-dimensional representation in the first part of our method. This serves well for two purposes. First, this allows the initial feature representation to be robust to slight changes in the spike-event pattern. This follows as an obvious effect of smoothing the spike-event count matrix and the PCA based dimensionality reduction step. Second, using a compressed representation allows the method to extract rudimentary feature matches based on simpler event-distribution statistics using the PCA . Note that we do not need the feature matches to be too precise and accurate in the subsequent step, as we expect the new SFA based learner to incorporate translational and other transformational robustness irrespective of the quality of the feature matches. This is evidenced by the result shown in Fig.~\ref{fig:all_accs_traffic} (b), where we find that using too many PCA projections to generate the feature matches proves detrimental to the final tracking accuracy. In fact using only 10 projections is sufficient to obtain robust SFA based projections, with near-best accuracy. 

There are a number of takeaways from the comprehensive analysis of tracking performance, in response to parameter variations, as shown with the traffic dataset. First, we find that the feature quality improves when using more spike-timing precision, but decreases when being too precise. For example, in  Fig.~\ref{fig:all_accs_traffic} (d), one observes up to a 10\% increase in performance when choosing 25 partitions (250 FPS), as opposed to using only 5 partitions (50 FPS). This indicates that the additional temporal precision helped. However, increasing beyond 100 partitions (1000 FPS) leads to a sudden drop in performance, which indicates that reading too much into the spike times (which can be up to 1 microsecond detail) is not advisable. A similar finding was reported in (ryad spike-timing precision), where spike-time precision beyond 1 ms negatively affected performance. 

Second, for the traffic dataset, we observe a steady performance improvement when using longer temporal windows in the construction of the spatiotemporal features. Longer time windows implies longer duration of motion, and therefore the features extracted will prioritize encoding motion information, when compared with shorter time windows. Clearly, since the motion of the feature points in the traffic data is usually consistent and smooth, prioritizing motion information will be helpful for a tracker to identify a feature point's trajectory (and thus do better feature matching). 

Third, the performance comparison analysis shown in \ref{fig:comparisons_babel} demonstrates that SFA based features perform considerably better than PCA based features. Note that our method finds projections which are informative and slowly changing across feature matches, as opposed to PCA, whose only criteria is to select the most informative projections. The additional \textit{slowness} aspect of our algorithm leads to the preferential learning of weights which have more stable responses across the spatiotemporal domain of events, and therefore are a more reliable identifier of a feature point. This it does by building feature tolerance to the local visual transformations present across the training data (Feature matching step).   

Evidence gathered from our experiments primarily shows that our feature learning algorithm is capable of learning specialized features towards the spatiotemporal patterns of motion present in the training data. In future, an important aspect of further investigation would be to replace the Gaussian smoothing step with smoothness constraints on the weights. Much like the motivation behind the gaussian smoothing, which is to retain topological smoothness of the projection values, these constraints will hope to achieve the same. In addition, such constraints might possibly preserve sharper transitions (edges) in the values of the weight matrices, which is not possible with the Gaussian smoothing step.  

The current methodology described in the paper essentially learns a mapping from the spike-event domain to a feature domain. For feature tracking purposes, this method needs the location of the feature points to be specified, in order to track them. To address this issue, such a method could be complemented with an additional corner point detection algorithm, like the one in \cite{fast_corner}, to make it self-sufficient, and more effective at generating feature matches. 

%
%

\section{Conclusion} \label{sec:conclusion}
A methodology for learning spatiotemporal features directly from event-data has been proposed. The feature learning criteria enforced are (i) smooth responses to individual spike-event perturbations, and (ii) slowly changing responses to change of spatio-temporal location attributing to re-appearance of feature points. We evaluated the obtained feature representations by the track-ability of feature points across multiple event-data recordings. The method outperforms other ways of obtaining spatiotemporal features, including time surface based representations. The features are found to adapt their robustness to the specific spatiotemporal transformations present in the training data (translation, rotation and scaling), while simultaneously showing good generalisation across datasets containing different motion patterns. Such desirable properties enable our method to find future applications in action recognition and feature matching applications.

\bibliographystyle{IEEEtran}
\bibliography{main}

%
%

%
\begin{IEEEbiography}[{\includegraphics[width=1in,height=1.25in,clip,keepaspectratio]{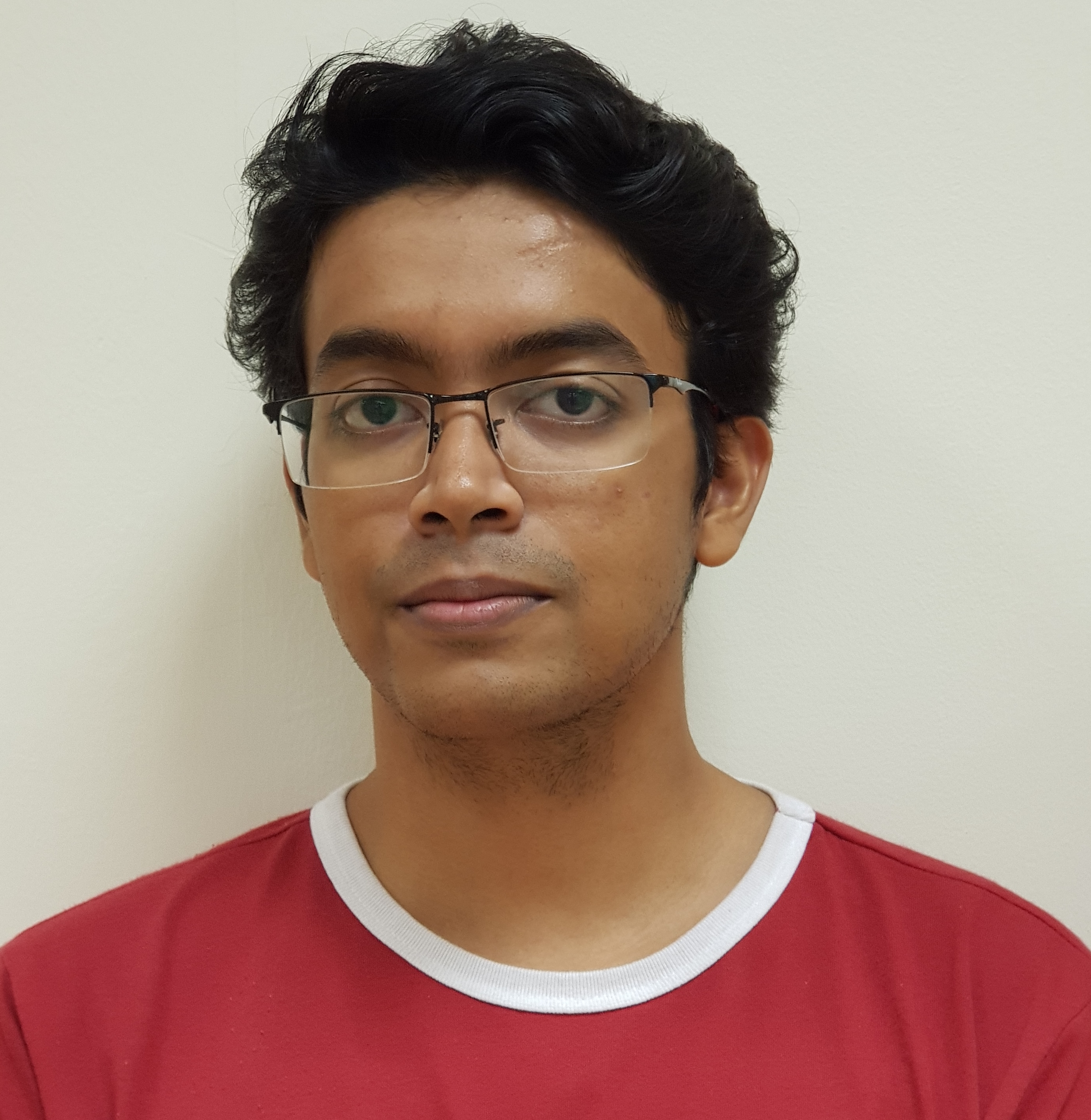}}]{Rohan Ghosh}
received the B.Tech. and M.Tech. degrees in Electronics and Electrical Communication Engineering from Indian Institute of Technology, Kharagpur, India in 2013 and 2014 respectively. He is currently working towards the PhD degree in Neuromorphic Vision in the Department of Electrical and Computer Engineering at National University of Singapore, Singapore. 
\end{IEEEbiography}
\begin{IEEEbiography}[{\includegraphics[width=1in,height=1.25in,clip,keepaspectratio]{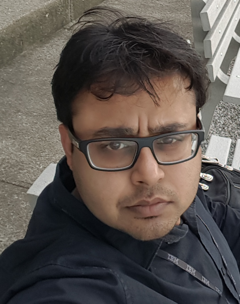}}]{Anupam K. Gupta}
received the B.Tech. degree in Mechanical Engineering from Institute of Engineering and Technology, Lucknow, India, in 2008, M.Sc. in Mechanical Engineering from Masachussetts Institute of Technology, Cambridge, USA, and Nanyang Technological University, Singapore in 2010 and PhD degree in Mechanical Engineering from Virginia Tech, Blacksburg, USA in 2016. He is currently a postdoctoral research fellow in the Singapore Institute for Neurotechnology (SINAPSE), National University of Singapore where his research focuses on bio-inspired sensing and algorithms for Vision and Tactile modalities. He has previously worked in IBM Research, Yorktown, N.Y., USA on bio-inspired speech recognition. 

\end{IEEEbiography}

\begin{IEEEbiography}[{\includegraphics[width=1in,height=1.25in,clip,keepaspectratio]{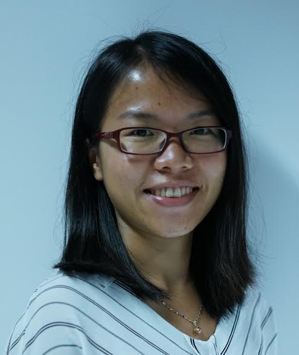}}]{Siyi Tang}
received the B.Eng. degree in Electrical Engineering from National University of Singapore in 2016. She was a recipient of National University of Singapore Science and Technology Undergraduate Scholarship.
\end{IEEEbiography}

\begin{IEEEbiography}[{\includegraphics[width=1in,height=1.25in,clip,keepaspectratio]{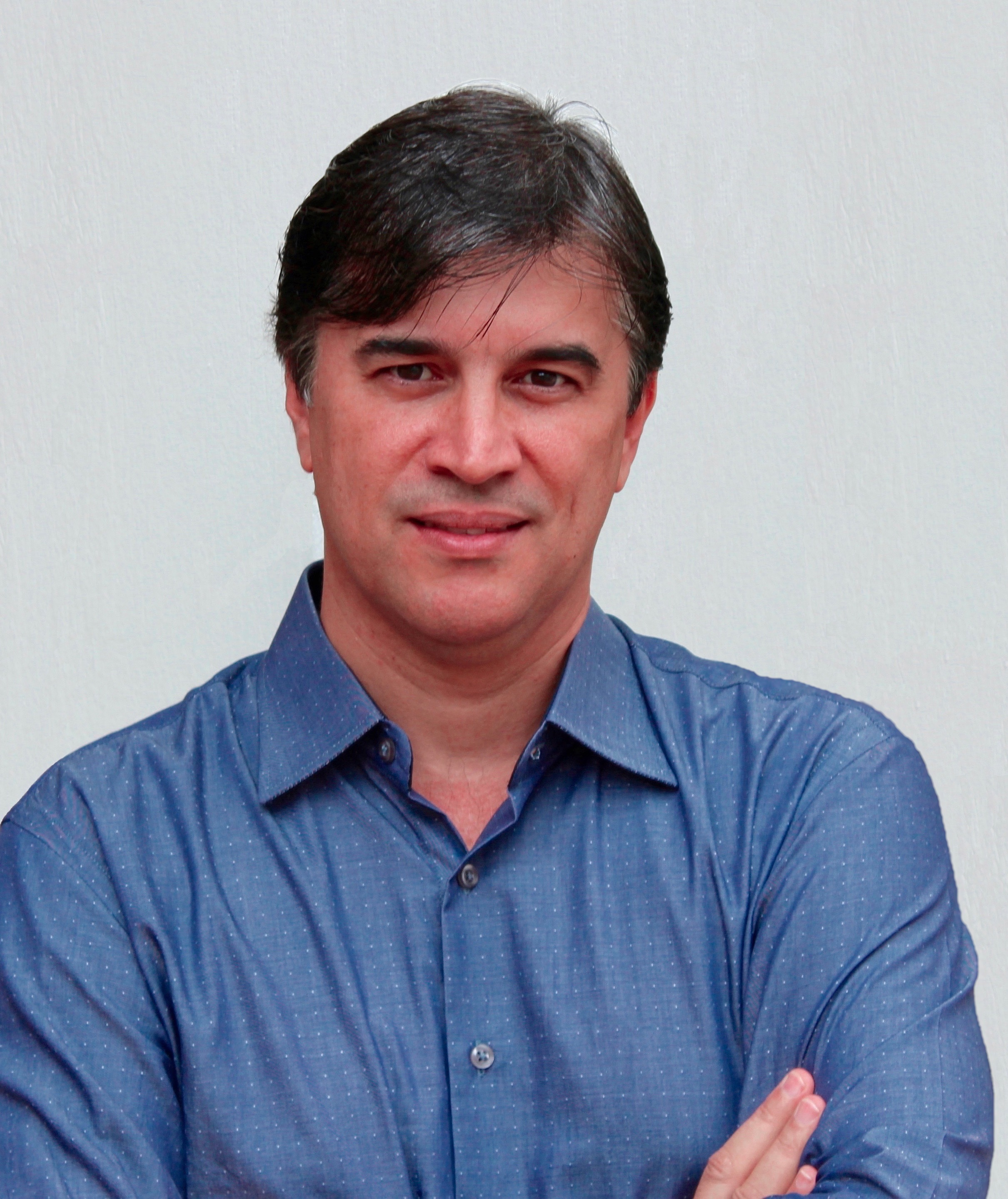}}]{Alcimar B. Soares}
Alcimar B. Soares received the B.S. degree in Electrical Engineering at the Federal University of Uberlândia, Brazil, in 1987, where he also concluded his MSc in the area of Artificial Intelligence in 1990. He received his PhD in Biomedical Engineering at the University of Edinburgh, UK, in 1997, and completed a research fellowship at the department of Biomedical Engineering of the Johns Hopkins University, USA, in 2014. Has was the Head of the Graduate Program of Electrical Engineering from 2002 to 2004, Head of the Faculty of Electrical Engineering from 2004 to 2008, and Dean for Research and Graduate Studies of the Federal University of Uberlândia from 2008 to 2012. He also was the Head of the Committees for the creation of the Biomedical Engineering Undergraduate and Graduate Programs of the Federal University of Uberlândia in 2005 and 2012, respectively. Dr. Soares is currently a Full Professor and Head of the Biomedical Engineering Lab at the Faculty of Electrical Engineering of the Federal University of Uberlândia, Brazil. He is the Editor-in-Chief of the Research on Biomedical Engineering journal, Deputy Editor of the Medical and Biological Engineering and Computing journal, Associate-Editor of the Journal of Biomedical Engineering and Biosciences and Associate-Editor of the journal Innovative Biomedical Technologies and Health Care. He is also member of various scientific societies, such as the IEEE Engineering in Medicine and Biology Society, the Brazilian Society of Biomedical Engineering, the Brazilian Society of Electromyography and Kinesiology and the International Society of Electromyography and Kinesiology. His research interests include modeling and estimation of neuromotor control systems, large scale neural systems dynamics, targeted neuroplasticity, decoding neural activity, brain machine interfaces and rehabilitation and assistive devices.
\end{IEEEbiography}


\begin{IEEEbiography}[{\includegraphics[width=1in,height=1.25in,clip,keepaspectratio]{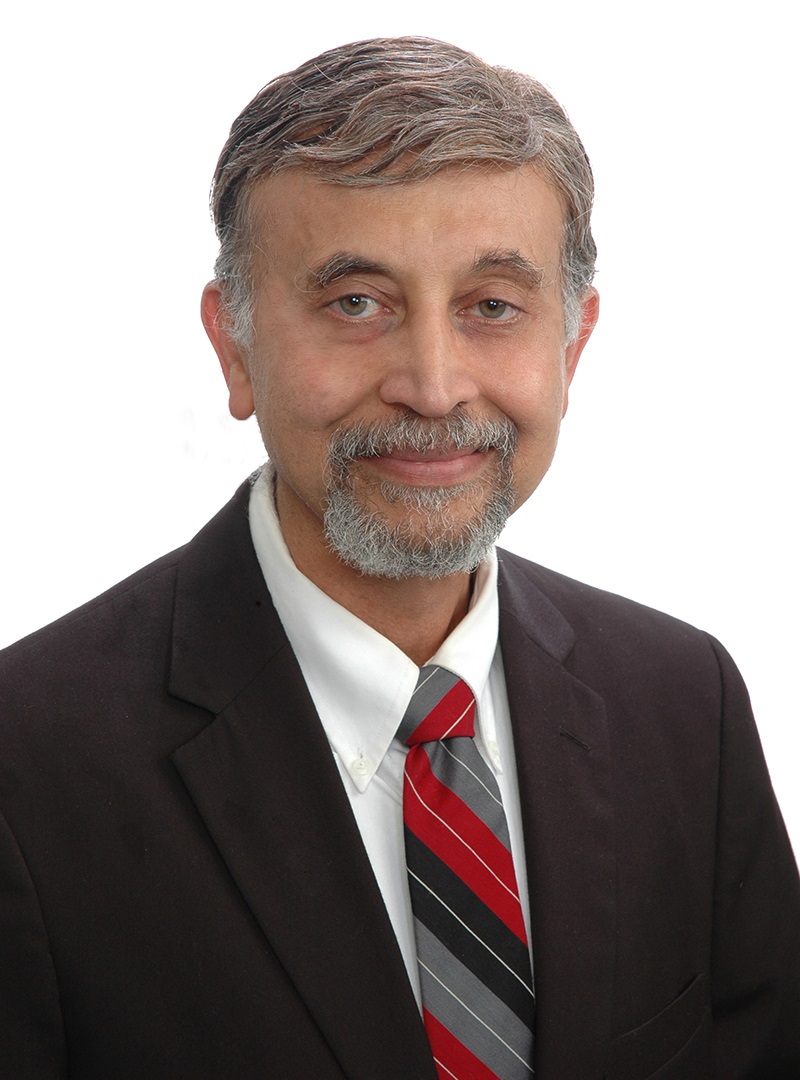}}]{Nitish V. Thakor}
is the Director of the Singapore Institute for Neurotechnology (SINAPSE) at the National University of Singapore, as well as the Professor of Biomedical Engineering at Johns Hopkins University in the USA.  Dr. Thakor's technical expertise is in the field of Neuroengineering, where he has pioneered many technologies for brain monitoring to prosthetic arms and neuroprosthesis.  He is the author of 340 refereed journal publications, 20  patents, and co-founder of 3 companies. He is currently the Editor in Chief of Medical and Biological Engineering and Computing, and was the Editor in Chief of IEEE TNSRE from 2005-2011. Dr. Thakor is a recipient of a Research Career Development Award from the National Institutes of Health and a Presidential Young Investigator Award from the National Science Foundation, and is a Fellow of the American Institute of Medical and Biological Engineering, IEEE, Founding Fellow of the Biomedical Engineering Society, and Fellow of International Federation of Medical and Biological Engineering.  He is a recipient of the award of Technical Excellence in Neuroengineering from IEEE Engineering in Medicine and Biology Society, Distinguished Alumnus Award from Indian Institute of Technology, Bombay, India, and a Centennial Medal from the University of Wisconsin School of Engineering.
\end{IEEEbiography}

\newpage

\appendix

\begin{proof}[Proof of Theorem \ref{thm:circconv}]
The given 3D matrices $\mathbf{A}$, $\mathbf{B}$, and $\mathbf{C}$ are zero padded to make their size equal. Using the associative property of convolutions we can write,
\begin{equation*}
\mathbf{A} \ast \mathbf{B} \ast \mathbf{C} = \mathbf{A} \ast \mathbf{C} \ast \mathbf{B}.
\end{equation*}

Let us consider 3D matrices $\mathbf{Y}_1,\mathbf{Y}_2,\mathbf{Y}_3$, each of size $(l_1 \times l_2 \times l_3)$, such that $\mathbf{Y}_3=\mathbf{Y}_1 \ast \mathbf{Y}_1$. The central element of $\mathbf{Y}_3$ can be simply written as

{\smaller
\begin{align*}
 \mathbf{Y}_3\left( \frac{l_1}{2}, \frac{l_2}{2}, \frac{l_3}{2}\right)&=\displaystyle\sum_{i,j,k}^{l_1,l_2,l_3} \mathbf{Y}_1\left( \frac{l_1}{2}-i, \frac{l_2}{2}-j, \frac{l_3}{2}-k\right)\mathbf{Y}_2\left(i,j,k\right) \\
 &=sum\left(\left(\sim \mathbf{Y}_1\right) \circ \mathbf{Y}_2\right) \numberthis  \label{eq:proof1}
\end{align*} 
}
where  $\left(\sim \mathbf{Y}_1\right)$ as the mirror image matrix of $\mathbf{Y}_1$, obeying $\left( \sim \mathbf{Y_1} \right)(i,j,k)=\mathbf{Y}_1\left( \frac{s_1}{2}-i, \frac{s_2}{2}-j, \frac{s_3}{2}-k\right)$. Observe that $\left(\sim (\sim \mathbf{Y}_1) \right)=\mathbf{Y}_1$. Conversely, $sum\left(\mathbf{Y}_1 \circ \mathbf{Y}_2\right)$ is the central element of the convolution  $(\sim \mathbf{Y}_1 ) \ast \mathbf{Y}_2$. 

Therefore, $sum\left(\mathbf{A} \circ (\mathbf{B} \ast \mathbf{C}) \right)$ is essentially the central element of the matrix $ \left( \sim \mathbf{A} \right) \ast (\mathbf{B} \ast \mathbf{C})$. Using the associative property of convolutions, this matrix can be rewritten as $ \left( \sim \mathbf{A} \ast \mathbf{C} \right)  \ast \mathbf{B}$. Applying~(\ref{eq:proof1}), its central element can simply be obtained as $sum\left(\left(\sim \left( \sim \mathbf{A} \ast \mathbf{C} \right) \right) \circ \mathbf{B} \right)$. 

It is trivial to show that $\sim(\mathbf{A} \ast \mathbf{B})=(\sim \mathbf{A}) \ast (\sim \mathbf{B})$. Using this property we continue with the previous steps as follows
\begin{align*}
sum\left(\mathbf{A} \circ (\mathbf{B} \ast \mathbf{C}) \right) &= sum\left(\left(\sim \left( \sim \mathbf{A} \ast \mathbf{C} \right) \right) \circ \mathbf{B} \right) \\
&=sum\left(\left( \left(\sim \sim \mathbf{A}\right) \ast \left( \sim \mathbf{C} \right) \right) \circ \mathbf{B} \right) \\
&=sum\left(\left( \mathbf{A}  \ast  \mathbf{C} \right) \circ \mathbf{B} \right) .
\end{align*} 
The last step substitutes $\sim \mathbf{C} =\mathbf{C}$ because $\mathbf{C}$ is the 3D Gaussian distribution kernel matrix. 

Note that this result holds for any $\mathbf{C}$ which is equal to its mirror matrix $\sim \mathbf{C}$. A more general result is 
\begin{equation}
sum\left(\mathbf{A} \circ (\mathbf{B} \ast \mathbf{C}) \right) =  sum\left(\left(\mathbf{A}  \ast  (\sim \mathbf{C}) \right) \circ \mathbf{B} \right).\qedhere
\end{equation} 
\end{proof}





\end{document}